%% file: main.tex
\def\isarxiv{1} %%% for icml submission version, we comment this line

\ifdefined\isarxiv
\documentclass[11pt]{article}

\usepackage[numbers]{natbib}

\else

\documentclass{article}

\usepackage{microtype}
\usepackage{graphicx}
\usepackage{subfig}
\usepackage{booktabs} % for professional tables

\usepackage{hyperref}

\usepackage{icml2025}

\usepackage{amsmath}
\usepackage{amssymb}
\usepackage{mathtools}
\usepackage{amsthm}

\usepackage[capitalize,noabbrev]{cleveref}

\theoremstyle{plain}
\newtheorem{theorem}{Theorem}[section]
\newtheorem{proposition}[theorem]{Proposition}
\newtheorem{lemma}[theorem]{Lemma}
\newtheorem{corollary}[theorem]{Corollary}
\theoremstyle{definition}
\newtheorem{definition}[theorem]{Definition}
\newtheorem{assumption}[theorem]{Assumption}
\theoremstyle{remark}
\newtheorem{remark}[theorem]{Remark}

\usepackage[textsize=tiny]{todonotes}

\icmltitlerunning{Fundamental Limits of Prompt Tuning VAR: Universality, Capacity and Efficiency}

\fi

\usepackage{amsmath}
\usepackage{amsthm}
\usepackage{amssymb}
\usepackage{subfig}
\usepackage{graphicx}
\usepackage{grffile}
\usepackage{wrapfig,epsfig}
\usepackage{url}
\usepackage{xcolor}
\usepackage{epstopdf}

\usepackage{bbm}
\usepackage{dsfont}

\allowdisplaybreaks

\ifdefined\isarxiv

\usepackage{tikz}
\usepackage{hyperref}  %%% arxiv don't allow this.
\hypersetup{colorlinks=true,citecolor=blue,linkcolor=blue} %%% Zhao : maybe we should comment this in submission.
\usetikzlibrary{arrows}
\usepackage[margin=1in]{geometry}

\else

\usepackage{microtype}
\usepackage{hyperref}
\definecolor{mydarkblue}{rgb}{0,0.08,0.45}
\hypersetup{colorlinks=true, citecolor=mydarkblue,linkcolor=mydarkblue}

\fi
 
\graphicspath{{./figs/}}

\theoremstyle{plain}
\newtheorem{theorem}{Theorem}[section]
\newtheorem{lemma}[theorem]{Lemma}
\newtheorem{definition}[theorem]{Definition}

\newtheorem{corollary}[theorem]{Corollary}

\newtheorem{assumption}[theorem]{Assumption}

\newtheorem{fact}[theorem]{Fact}
\newtheorem{remark}[theorem]{Remark}

\newcommand{\wh}{\widehat}

\newcommand{\N}{\mathcal{N}}
\newcommand{\R}{\mathbb{R}}

\renewcommand{\d}{\mathrm{d}}

\newcommand{\VAR}{\mathrm{VAR}}
\newcommand{\X}{\mathsf{X}}
\newcommand{\Y}{\mathsf{Y}}
\newcommand{\Z}{\mathsf{Z}}
\newcommand{\F}{\mathsf{F}}
\newcommand{\V}{\mathsf{V}}

\DeclareMathOperator*{\E}{{\mathbb{E}}}

\DeclareMathOperator{\diag}{diag}

\makeatletter
\newcommand*{\RN}[1]{\expandafter\@slowromancap\romannumeral #1@}
\makeatletter
\usepackage{lineno}

\begin{document}

\ifdefined\isarxiv

\date{}

\title{Universal Approximation of Visual Autoregressive Transformers}
\author{
Yifang Chen\thanks{\texttt{
yifangc@uchicago.edu}. The University of Chicago.}
\and
Xiaoyu Li\thanks{\texttt{
xiaoyu.li2@student.unsw.edu.au}. University of New South Wales.}
\and
Yingyu Liang\thanks{\texttt{
yingyul@hku.hk}. The University of Hong Kong. \texttt{
yliang@cs.wisc.edu}. University of Wisconsin-Madison.} 
\and
Zhenmei Shi\thanks{\texttt{
zhmeishi@cs.wisc.edu}. University of Wisconsin-Madison.}
\and 
Zhao Song\thanks{\texttt{ magic.linuxkde@gmail.com}. The Simons Institute for the Theory of Computing at UC Berkeley.}
}

\else

\twocolumn[
\icmltitle{Universal Approximation of Visual Autoregressive Transformers}



\icmlsetsymbol{equal}{*}

\begin{icmlauthorlist}
\icmlauthor{Firstname1 Lastname1}{equal,yyy}
\icmlauthor{Firstname2 Lastname2}{equal,yyy,comp}
\icmlauthor{Firstname3 Lastname3}{comp}
\icmlauthor{Firstname4 Lastname4}{sch}
\icmlauthor{Firstname5 Lastname5}{yyy}
\icmlauthor{Firstname6 Lastname6}{sch,yyy,comp}
\icmlauthor{Firstname7 Lastname7}{comp}
\icmlauthor{Firstname8 Lastname8}{sch}
\icmlauthor{Firstname8 Lastname8}{yyy,comp}
\end{icmlauthorlist}

\icmlaffiliation{yyy}{Department of XXX, University of YYY, Location, Country}
\icmlaffiliation{comp}{Company Name, Location, Country}
\icmlaffiliation{sch}{School of ZZZ, Institute of WWW, Location, Country}

\icmlcorrespondingauthor{Firstname1 Lastname1}{first1.last1@xxx.edu}
\icmlcorrespondingauthor{Firstname2 Lastname2}{first2.last2@www.uk}

\icmlkeywords{Machine Learning, ICML}

\vskip 0.3in
]



\printAffiliationsAndNotice{\icmlEqualContribution} % otherwise use the standard text.

\fi

\ifdefined\isarxiv
\begin{titlepage}
  \maketitle
  \begin{abstract}
\input{0_abstract}

  \end{abstract}
  \thispagestyle{empty}
\end{titlepage}

{\hypersetup{linkcolor=black}
\tableofcontents
}
\newpage

\else

\begin{abstract}
\input{0_abstract}
\end{abstract}

\fi

\input{1_intro} %%% Section 1. Introduction
\input{2_related_work}

\input{3_preli}
\input{4_single}
\input{5_var}

\input{6_flowar}
\input{7_conclusion}

\ifdefined\isarxiv

\else
\bibliography{ref}
\bibliographystyle{icml2025}

\fi

\newpage
\onecolumn
\appendix
\begin{center}
	\textbf{\LARGE Appendix }
\end{center}

%%%% Cut-line between first 10 pages and appendix

\input{8_app_prelim}

\input{9_app_var}

\input{10_app_flowar}

\ifdefined\isarxiv
%\section*{Acknowledgments}
\bibliographystyle{alpha}
\bibliography{ref}
\else

\fi

%%% some writing rules

%% Writing rule for creating tags.
%% Tags :
%% Theorem    \ref{thm:bla_bla}
%% Lemma      \ref{lem:bla_bla}
%% Claim      \ref{cla:bla_bla}
%% Corollary  \ref{cor:bla_bla}
%% Fact       \ref{fac:bla_bla}
%% Definition \ref{def:bla_bla}
%% Section    \ref{sec:bla_bla}
%% Subsection \ref{sub:bla_bla}
%% Equation   \ref{eq:bla_bla}

\end{document}

%% file: 0_abstract.tex
We investigate the fundamental limits of transformer-based foundation models, extending our analysis to include Visual Autoregressive (VAR) transformers. VAR represents a big step toward generating images using a novel, scalable, coarse-to-fine ``next-scale prediction'' framework. These models set a new quality bar, outperforming all previous methods, including Diffusion Transformers, while having state-of-the-art performance for image synthesis tasks. Our primary contributions establish that, for single-head VAR transformers with a single self-attention layer and single interpolation layer, the VAR Transformer is universal. From the statistical perspective, we prove that such simple VAR transformers are universal approximators for any image-to-image Lipschitz functions. Furthermore, we demonstrate that flow-based autoregressive transformers inherit similar approximation capabilities. Our results provide important design principles for effective and computationally efficient VAR Transformer strategies that can be used to extend their utility to more sophisticated VAR models in image generation and other related areas.

%% file: 1_intro.tex
\section{Introduction}
Transformer-based architectures have reshaped the landscape of modern machine learning, demonstrating state-of-the-art performance across a wide range of tasks, including natural language processing (e.g., GPT-o3~\cite{gpto1}, Llama 3.3~\cite{llama3_arxiv,llama3_blog}, and Claude 3.5~\cite{claude3_pdf}), computer vision, and generative modeling. Their core mechanism of self-attention~\cite{vsp+17} allows for effective modeling of long-range dependencies in data, positioning transformers as a cornerstone of contemporary deep learning research. 
One particularly compelling variant is the Visual AutoRegressive (VAR) Transformer~\cite{tjy+24}, which adapts the transformer paradigm to structured image synthesis. By employing a coarse-to-fine ``next-scale prediction'' approach, VAR Transformers produce high-quality images more efficiently than many standard diffusion-based methods~\cite{sme20}. This iterative, pyramid-like generation process has demonstrated strong performance on large-scale visual tasks, indicating that multi-scale attention can capture hierarchical features in an image. Yet, despite promising empirical evidence, the theoretical underpinnings of VAR Transformers—specifically, whether they inherit the well-established universal approximation properties of classical transformers—remain an open question.

In parallel to VAR Transformers, flow-based generative methods (e.g., real-valued non-volume preserving (RealNVP) and Glow) have also garnered attention for their ability to generate high-fidelity samples in an invertible and tractable manner. Recent efforts have integrated autoregressive decompositions with flow-based designs, giving rise to Flow AutoRegressive (FlowAR) \cite{ryh+24} architectures. These models aim to blend the interpretability and stability of normalizing flows with the powerful representation learning of autoregressive transformers, potentially yielding more robust and scalable training dynamics. However, despite promising practical results, the theoretical investigation into the representational power of FlowAR remains similarly sparse.

This paper addresses two central gaps in our theoretical understanding of generative transformer architectures:
\begin{enumerate}
    \item {\bf Universality of VAR Transformers:} Although classic transformers are known to approximate arbitrary sequence-to-sequence functions~\cite{ks24,kkm22,ybr+20,hwg+24}, the additional pyramid up-sampling layers in VAR Transformers modify the input-output structure in ways that have not been theoretically dissected. We aim to rigorously determine whether VAR Transformers can still achieve universal approximation while relying on multi-scale token representations.
    \item {\bf Universality of FlowAR:} Normalizing-flow-inspired architectures equipped with autoregressive attention mechanisms promise both efficient sampling and tractable likelihood estimates. Yet, their approximation capabilities have not been formally established. Can a FlowAR model approximate arbitrary continuous transformations with any desired precision?
\end{enumerate}

In bridging these gaps, we seek to provide a unified view of how up-sampling, attention, and flow-based operations interact within transformer architectures to yield expressive function classes. Our contributions can be described as follows.
\begin{itemize}
    \item {\bf Universality of single-layer, single-head VAR Transformers (see Theorem~\ref{thm:var_universality}).} We establish that even minimal VAR Transformer designs can approximate any Lipschitz sequence-to-sequence function arbitrarily closely, extending classical universality results to the VAR setting. Our theoretical analysis demonstrates that the coarse-to-fine up-sampling process, in concert with self-attention, confers enough expressive power to realize complex transformations.
    
    \item {\bf Universality of FlowAR (see Corollary~\ref{corollary:flowar_universality}).} We further show that flow-based autoregressive transformers inherit similar approximation capabilities. In particular, we prove that FlowAR models can capture arbitrary Lipschitz transformations, illustrating how flow-based transformations and autoregressive attention can be combined without sacrificing the universality guarantees traditionally associated with transformers.
\end{itemize}

Our proofs highlight the complementary roles of self-attention, up-sampling layers, and flow-based transformations. We elucidate how these components—when arranged in a pyramid fashion (for VAR) or in an invertible mapping (for FlowAR)—enable powerful function approximation. These insights suggest a pathway for designing more flexible and efficient generative models that balance computational constraints with representational breadth.

By elucidating the theoretical foundations of both VAR Transformers and FlowAR, we advance a deeper understanding of why these models perform so effectively in practice. Our results reinforce that efficiency and expressiveness need not be at odds: even with seemingly minimal configurations (e.g., single-layer, single-head), these architectures can approximate highly complex functions. Moreover, our formal proofs set a stage for further explorations into the trade-offs between model depth, head multiplicity, and approximation efficiency, as well as studying domain-specific constraints in images, text, or other structured data formats.

\paragraph{Roadmap.} 
In Section~\ref{sec:related_work}, we survey related work on transformer universality and generative architectures. Section~\ref{sec:prelim} introduces the necessary background on VAR Transformers, focusing on up-interpolation layers and the structure of the VAR Transformer block. Section~\ref{sec:context_map} analyzes the contextual mapping properties of attention in VAR Transformers. Sections~\ref{sec:var_mainresult} present our main universality theorems, demonstrating that even simple VAR Transformers approximate Lipschitz functions with arbitrary accuracy. Section~\ref{sec:flowar_universality} provides universality results for FlowAR, a related autoregressive variant. Finally, Section~\ref{sec:conclusion} concludes with a summary of our findings and reflections on future directions. 

%% file: 2_related_work.tex
\section{Related Work}\label{sec:related_work}

\paragraph{AutoRegressive Models.} AutoRegressive models for visual generation \cite{dyh+21,dzht22} process 2D images by converting them into 1D token sequences. Early approaches, such as PixelCNN \cite{vke+16} and PixelSNAIL \cite{cmr+18}, introduced pixel-by-pixel image generation using a raster-scan order. Later advancements \cite{rvav19,ero21,lkk+22} adapted this idea to generate image tokens following a similar raster sequence. For instance, VQ-GAN \cite{ero21} utilizes a decoder-only transformer akin to GPT-2 for generating images, while VQVAE-2 \cite{rvav19} and RQ-Transformer \cite{lkk+22} enhance the method by incorporating hierarchical scales or stacked representations. Recently, Visual AutoRegressive ($\VAR$) modeling \cite{tjy+24} proposed an innovative coarse-to-fine ``next-scale prediction'' strategy, significantly improving scalability, inference speed, and image quality, thus surpassing conventional autoregressive models and diffusion transformers.

\paragraph{Diffusion Models.} Diffusion models \cite{hja20,rbl+22} excel in generating high-resolution images by iteratively refining noise into coherent visuals. Prominent examples, such as DiT \cite{px23} and U-ViT \cite{bnx+23}, leverage probabilistic frameworks to learn data distributions effectively. Recent progress in diffusion-based image generation has focused on enhancing sampling techniques and training efficiency \cite{se19,sme20,lzb+22,hwl+24,cgl+25_high,ssz+25_dit}, advancing latent-space learning \cite{rbl+22,wsd+24,wxz+24,lzw+24}, refining model architectures \cite{hsc+22,px23,lsss24,wcz+23,xsg+24}, and exploring applications in 3D generation \cite{pjbm22,wlw+24,xlc+24,cgl+25_text}.

\paragraph{Universality of Transformers.}
The universality of transformers refers to their capacity to function as universal approximators, meaning they can theoretically model any sequence-to-sequence function with arbitrary precision. \cite{ybr+20} establish this property by demonstrating that transformers achieve universal approximation through the stacking of multiple layers of feed-forward and self-attention functions. Taking a different approach, \cite{jl23} confirms transformer universality by leveraging the Kolmogorov-Albert representation theorem. Additionally, \cite{adtk23} extend the universality results to architectures featuring non-standard attention mechanisms. More recently, \cite{ks24} shows that even a transformer with a single self-attention layer can act as a universal approximator. Of independent interest, \cite{hl24} explore the generalization and approximation properties of transformers under assumptions of Hölder smoothness and low-dimensional subspaces. \cite{lll+25_loop} showed that the looped transformer can approximate the Hypergraphs algorithm.

%% file: 3_preli.tex
\section{Preliminary}\label{sec:prelim}
In this section, we introduce the fundamental definitions of our work. 
In Section~\ref{sec:notations}, we introduce all related math notations used in this paper.
In Section~\ref{sec:var_phase1}, we introduce the components in phase one of the VAR Model. 
In Section~\ref{sec:var_transformer}, we mathematically detail the VAR Transformer blocks. 

\subsection{Notations}\label{sec:notations}

We denote the $\ell_p$ norm of a vector $x$ by $\| x \|_p$, i.e., $\|x\|_1 := \sum_{i=1}^n |x_i|$, $\| x \|_2 := (\sum_{i=1}^n x_i^2)^{1/2}$ and $\| x \|_{\infty} := \max_{i \in [n]} |x_i|$. For a vector $x \in \R^n$, $\exp(x) \in \R^n$ denotes a vector where $\exp(x)_i$ is $\exp(x_i)$ for all $i \in [n]$. For $n > k$, for any matrix $A \in \R^{n\times k}$, we denote the spectral norm of $A$ by $\| A \|$, i.e., $\| A \| := \sup_{x\in \R^k} \| Ax \|_2 / \| x \|_2$. We define the function norm as $\| f \|_\alpha :=  (\int \| f(X) \|_\alpha^\alpha \d X)^{1/\alpha}$ where $f$ is a function. For a matrix $X \in \R^{n_1 n_2 \times d}$, we use $\X \in \R^{n_1 \times n_2 \times d}$ to denote its tensorization, and we only assume this for letters $X$ and $Y$.

\subsection{VAR Phase One}\label{sec:var_phase1} \label{sec:phase_1}

We first present Phase One of VAR model based on~\cite{kll+25}.

The VAR model uses the VAR Transformer to convert the initialized token map $X_{\mathrm{init}}$ into a series of pyramid-shaped token maps. 
The VAR Transformer alternates between up sample blocks and attention layers to get the output.

\paragraph{Up Sample Blocks.} 
The $k$-th up sample block takes as input the initial token map $X_{\mathrm{init}}$ and the previous pyramid-shaped token maps $X_1, \ldots, X_k$, sets $Y_1 = X_{\mathrm{init}}$ and up samples each $X_i$ into a new token map $Y_{i+1}$, and outputs the new pyramid-shaped token maps $Y_1, \ldots, Y_{k+1}$.  

The upsampling on each token map $X_r (r \in [k])$ uses interpolation with a bicubic spline kernel.

% We first define some necessary operations used in VAR. 
\begin{definition}[Bicubic Spline Kernel, Definition 3.1 from~\cite{kll+25} on Page 7 ]\label{def:bi_spline_kernel}
    A bicubic spline kernel is a piecewise cubic function $W: \R \to \R$ that satisfies $W(x) \in [0,1]$ for all $x \in \R$.
\end{definition}

\begin{definition}[Up-interpolation Layer for One-Step Geometric Series]\label{def:up_inter_layer_one_step}
The Up-Interpolation layer is defined as follows:
\begin{itemize}
    \item Let $r \geq 2$ be an integer.
    \item Let $h_{r-1} < h_{r}$ denote two positive integers
    \item Let $w_{r-1} < w_{r}$ denote two positive integers.
    
    \item Let $d \in \mathbb{N}$ denote the number of channels.
    \item Let $\X \in \R^{h_{r-1} \times w_{r-1} \times d}$ denote the input feature map.
    \item Let $\mathsf{Y} \in \R^{h_{r} \times w_{r} \times d}$ denote the output feature map.
    \item Let $s,t \in \{-1,0,1,2\}$.
    \item Let $W: \R \to \R$ be a bicubic spline kernel as defined in~\ref{def:bi_spline_kernel}.
    
    We use $\phi_{{\rm up},r}: \R^{ h_{r-1} \times w_{r-1} \times c} \to \R^{h_{r} \times w_{r} \times c}$ to denote the up-interpolation operation then we have $\mathsf{Y} = \phi_{{\rm up}, r}(X)$. Specifically, for $i \in [h_{r}], j \in [w_{r}], l \in [c]$, we have
    \begin{align*}
        \mathsf{Y}_{i,j,l} := \sum_{s=-1}^2 \sum_{t=-1}^2 W(s) \cdot \mathsf{X}_{\frac{i \cdot h_{r-1}}{h_{r} }+s,\frac{j \cdot w_{r-1}}{w_{r}}+t,l} \cdot  W(t)
    \end{align*}
    
\end{itemize}
\end{definition}

After defining the Up-Interpolation Layer for a one-step geometric sequence, we can construct a Pyramid Up-Interpolation Layer, which applies multiple up-interpolation layers to generate token maps at different resolutions. Specifically, we can describe this Pyramid Up-Interpolation Layer through the following definition:

\begin{definition}[Pyramid Up-Interpolation Layer $\Phi_{{\rm}}$, $r=1$ Case]\label{def:pyramid_up_interpolation_layer_r1}
The Pyramid Up-Interpolation layer is defined as follows:
\begin{itemize}
    \item Let $d > 0$ denote one positive integer.
    \item  Let $\X_{\mathrm{init}} \in \R^{1\times 1 \times d}$ denote the initial token map.
\end{itemize}
We use $\Phi_{{\rm up}, 1} : \R^{ 1 \times 1 \times d} \to \R^{1 \times 1 \times d}$ such that
\begin{itemize}
    \item $\Phi_{\mathrm{up}, 1}(\X_{\mathrm{init}}) = \X_{\mathrm{init}}$.
\end{itemize}
\end{definition}

\begin{definition}[Pyramid Up-Interpolation Layer $\Phi_{{\rm}}$, $r \geq 2$ Case]\label{def:pyramid_up_interpolation_layer}
The Pyramid Up-Interpolation layer is defined as follows:
\begin{itemize}
    \item Let $d > 0$ denote one positive integer.
    \item Let $r \geq 2$.
    \item Let $\phi_{{\rm up},r}: \R^{ h_{r-1} \times w_{r-1} \times d} \to \R^{h_{r} \times w_{r} \times d}$ be defined in Definition~\ref{def:up_inter_layer_one_step}.
    \item  Let $\X_{\mathrm{init}} \in \R^{1\times 1 \times d}$ denote the initial token map.
\end{itemize}
We use $\Phi_{{\rm up}, r} : \R^{ h_{[r-1]} \times w_{[r-1]} \times d} \to \R^{h_{[r]} \times w_{[r]} \times d}$ such that
\begin{itemize}
    \item For all the $i \in [r] \setminus \{1\}$, we set $\mathsf{Y}_{i} = \phi_{{\rm up}, {i-1}} ( \X_{i-1} )$ (Here $\mathsf{Y}_i$ is the $i$-th layer of $Y$)
    \item For $i = 1$, we set $\mathsf{Y}_1 = \X_{\mathrm{init}}$
\end{itemize}
\end{definition}

\begin{figure}[!ht]\label{fig:iteration_trajectory}
    \centering
    \includegraphics[width=0.6\linewidth]{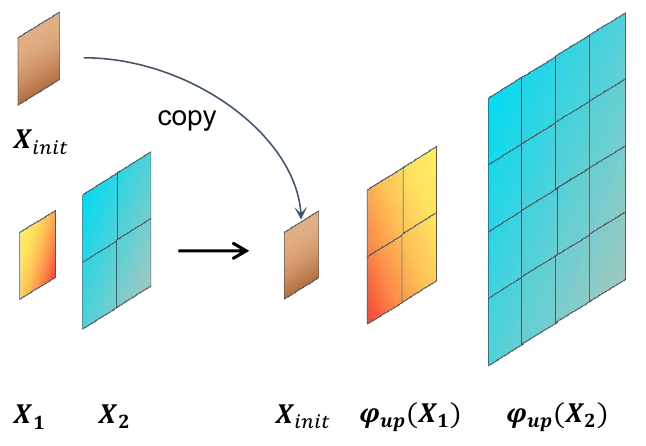}
    \caption{ One Pyramid Up-Interpolation Layer Instance $\Phi_{{\rm up},2}$, From Figure 1 in~\cite{kll+25}.}
    \label{fig:trajectory}
\end{figure}

\begin{remark}
    We have a pyramid-shaped token maps of size $ h_{[r+1]} \times w_{[r+1]} \times d $. To input this into the $\VAR$ Transformer, we merge the first two dimensions, transforming it into an input of size $ (\sum_{i=1}^{r+1} h_i w_i) \times d $. 
\end{remark}

Now, we are ready to introduce the VAR transformer. 
\begin{definition}[$\VAR$ Transformer]\label{def:var_transformer}
We define $\VAR$ transformer as follows:
\begin{itemize}
    \item Assume the $\VAR$ transformer has $m$ transformer layers.
    \item At the $i$-th transformer layer, let $g_i$ denote components excluding the attention layer, such as the LN layer or MLP layer.
    
    \item Let $\Phi_{\mathrm{up},r}$ denote the pyramid up-interpolation layer defined in Definition~\ref{def:pyramid_up_interpolation_layer}.
    \item  Let $\mathsf{Attn}_i$ stand for the self-attention layer, which is defined in Definition~\ref{def:single_layer_transformer}.
    \item Let $\X_{\mathrm{init}} \in \R^{1 \times 1 \times d}$ be an input token map and $X_{\mathrm{init}} \in \R^{1 \times d}$ be its matrix version.
    \item Let $n = \sum_{i=1}^m h_i w_i$.
\end{itemize}
     We define a $\VAR$ transformer as the following
\begin{align*}
    \mathsf{TF}(\X_{ \mathrm{init} }) := 
    & ~ g_m \circ \mathsf{Attn}_m \circ \Phi_{ \mathrm{up}, m} \circ \cdots \circ g_2 \circ \mathsf{Attn}_2 \circ \Phi_{ \mathrm{up}, 2} \\
    & ~ \circ g_1 \circ \mathsf{Attn}_1 \circ \Phi_{ \mathrm{up}, 1} (\X_{ \mathrm{init} }) \in \R^{n \times d},
\end{align*}
In this expression, $\circ$ stands for functional composition.
\end{definition}

\subsection{VAR Transformer Blocks}\label{sec:var_transformer}

Recall we have defined $\phi_{\mathrm{up}}: \R^{h \times w \times c} \to \R^{h' \times w' \times c}$ in Definition~\ref{def:up_inter_layer_one_step}. Since there is no non-linear operation in $\phi_{\mathrm{up}}$, $\phi_{\mathrm{up}}$ is equivalent to a matrix multiplication operation, where the dimension of the matrix is $\R^{h'w' \times hw}$. For simplicity, we view $\phi_{\mathrm{up}}$ as a $\R^{h'w' \times hw}$ dimension matrix in the following proofs.

\begin{remark} [Applying $\phi_{\mathrm{up}}$ on $X \in \R^{n \times d}$, Remark 4.8 from~\cite{kll+25} on Page 8]
The actual input of VAR Transformer Layer are $r$ input token maps, $X_1 \in \R^{h_1 \times w_1 \times d}, \ldots, X_r \in \R^{h_r \times w_r \times d}$. We denote them as $X \in \R^{n \times d}$, where $n := \sum_{i = 1}^r h_i w_i$. We denote $\phi_{\mathrm{up}}(X) \in \R^{n' \times d}$ as applying $\phi_{\mathrm{up}}$ to each $X_i \in \R^{h_i \times w_i \times d}$ for $i \in [r]$, where $n' = \sum_{i=1}^r h_i' w_i'$. 
\end{remark}

Then, we can combine multiple attention layers with other components (up-interpolation layers, multilayer perceptron layers, layer-wise normalization layers) to create a complete VAR Transformer architecture.

\begin{definition}[Single VAR Transformer Layer, Definition 4.9 from~\cite{kll+25} on Page 9]\label{def:var_transformer_single_layer}
We define a VAR transformer block as the following.
\begin{itemize}
    \item Assume the VAR transformer has $m$ Transformer layers.
    \item Let $\mathsf{FFN}$ denotes a single Feed-forward Layer (see Definition~\ref{def:ffn}).
    \item Let $\mathsf{Attn}$ stands for a single self-attention layer (see Definition~\ref{def:single_layer_transformer}).
\end{itemize}
\begin{align*}
    \mathsf{TF}_{\mathrm{var}}(X) = \mathsf{FFN} \circ \mathsf{Attn} \circ \phi_{\rm up} \in \R^{n \times d},
\end{align*}

In this expression, $\circ$ stands for functional composition.
\end{definition}

Now, we present the VAR Transformer Network Function Class.

\begin{definition}[VAR Transformer Network Function Class]\label{def:var_function_class} We define VAR Transformer Network Function Class as follows.
    \begin{itemize}
        \item Assume the VAR transformer network has $m$ layers.
        \item for $i \in [m]$, $\mathsf{FFN}_i$ denotes the Feed-forward at $i$-th layer (see Definition~\ref{def:ffn}), $\mathsf{Attn}_i$ denotes the Attention at $i$-th layer (see Definition~\ref{def:single_layer_transformer}), and $\phi^{\mathrm{up}}_i$ denotes the Up interpolation at $i$-th layer (see Definition~\ref{def:up_inter_layer_one_step}).
        \item Let $\mathcal{T}^{a,s,c}$ denote the VAR transformer network function class
        \item each function $\tau \in \mathcal{T}^{a,s,c}$ consists of VAR transformer blocks $\mathsf{TF}_{\mathrm{var}}^m$ with $a$ heads of size $s$ and $c$ MLP hidden neurons
    \end{itemize}
    \begin{align*}
        \mathcal{T}^{a,s,c} := 
        & ~ \{ \tau: \R^{n\times d} \to \R^{n\times d} | \tau = \mathsf{TF}_{\mathrm{var}}^m \circ \mathsf{TF}_{\mathrm{var}}^{m-1} \circ \hdots \circ \mathsf{TF}_{\mathrm{var}}^1 (X)
        \} 
    \end{align*}
\end{definition}

%% file: 4_single.tex
\section{Any-Rank Single-Layer Attention is a Contextual Mapping Function}
\label{sec:context_map}
In this section, we show that Attention is a contextual mapping function. 
In Section~\ref{sec:cont_map}, we give the definition of contextual mapping.
In Section~\ref{sec:any_rank_attn_contmap}, we introduce any-rank single-layer attention as a contextual mapping function.

\subsection{Contextual Mapping}\label{sec:cont_map}

{\bf Contextual Mapping.}
Let $X, Y \in \R^{n\times d}$ be the input embeddings and output label sequences, respectively. 
Let $X_{i} \in \R^{d}$ be the $i$-th token of each $X$ embedding sequence. 

\begin{definition}[Vocabulary, Definition 2.4 from~\cite{hwg+24} on Page 8]\label{def:vocab}
    We define the vocabulary.
    \begin{itemize}
        \item We define the $i$-th vocabulary set for $i \in [N]$ by $\mathcal{V}^{(i)}=\cup_{k \in[n]} X_k^{(i)} \subset \R^{d}$.
        \item We define the whole vocabulary set $\mathcal{V}$ as $\mathcal{V}=\cup_{i \in[N]} \mathcal{V}^{(i)} \subset \R^{d}$.
    \end{itemize}
\end{definition}
Note that while ``vocabulary'' typically refers to the tokens' codomain, here, it refers to the set of all tokens within a single sequence.
To facilitate our analysis, 
we introduce the idea of input token separation following~\cite{ks24,kkm22,ybr+20}.

\begin{definition}[Tokenwise Separateness, Definition 2.5 from~\cite{hwg+24} on Page 8] \label{def:token_seperate_new}
    We define the tokenwise separateness as follows.
    \begin{itemize}
        \item Let $X^{(1)}, \hdots, X^{(N)} \in \R^{n\times d}$ be embeddings. 
        \item Let $N$ be the number of sequences in the datasets.
        \item Let $n$ be the length of a sequence. i.e. $X^{(i)} \in \R^{n\times d}$
    \end{itemize}
    
    First, we state three conditions for $X^{(1)}, \hdots, X^{(N)}$ 
    \begin{itemize}
        \item [(i)] For any $i \in[N]$ and $k \in[n],\| X_k^{(i)}\|_2 > \gamma_{\min }$ holds. 
        \item [(ii)]
        For any $i \in[N]$ and $k \in[n], \|X_k^{(i)}\|_2<\gamma_{\max }$ holds.
        \item [(iii)]
        For any $i, j \in [N]$ and $k, l \in [n]$ if $X_k^{(i)} \neq X_ l^{(j)}, $ then $ \|X_k^{(i)}- X_l^{(j)}\|_2 > \delta$ holds.
    \end{itemize}
    Second, we define three types of separateness as follows, 
    \begin{itemize}
        \item {\bf Part 1.} If all conditions hold, then we call it  tokenwise $(\gamma_{\min }, \gamma_{\max }, \delta)$-separated
        \item {\bf Part 2.} If conditions (ii) and (iii) hold, then we denote this as $(\gamma, \delta)$-separateness.
        \item {\bf Part 3.} If only condition (iii) holds, then we denote it as $(\delta)$-separateness.
    \end{itemize} 
\end{definition}
To clarify condition (iii), we consider cases where there are repeated tokens between different input sequences. 
Next, we define contextual mapping. 
Contextual mapping describes a function's ability to capture the context of each input sequence as a whole and assign a unique ID to each input sequence.

\begin{definition} [$(\gamma,\delta)$-Contextual Mapping, Definition 2.6 from~\cite{hwg+24} on Page 8] \label{def:contextual_mapping_new}
A function $q:\R^{n\times d} \to \R^{n\times d}$ is said to be a $(\gamma, \delta)$-contextual mapping for a set of embeddings  $X^{(1)}, \hdots, X^{(N)} \in \R^{n\times d}$,  if the following conditions hold:
    \begin{itemize}
        \item {\bf Contextual Sensitivity $\gamma$.}
        For any $i \in[N]$ and $k \in[n],  \|q (X^{(i)})_k\|_2 < \gamma$ holds.
    
        \item {\bf Approximation Error $\delta$.}
        For any $i, j \in[N]$ and $k, l \in[n]$ such that $\mathcal{V}^{(i)} \neq \mathcal{V}^{(j)}$ or $X_k^{(i)} \neq X_l^{(j)}$, $\|q(X^{(i)})_k-q(X^{(j)})_l\|_2 > \delta$ holds.
    \end{itemize}
In addition, Note that $q (X^{(i)})$ for $i \in[N]$ is called a context ID of $X^{(i)}$.    
\end{definition}

\subsection{Any-Rank Single-Layer Attention is a Contextual Mapping Function}\label{sec:any_rank_attn_contmap}

Now we present the result showing that a softmax-based $1$-head, $1$-layer attention block with any-rank weight matrices is a contextual mapping.

\begin{lemma}[Any-Rank Attention as a $(\gamma, \delta)$-Contextual Mapping, Lemma 2.2 from~\cite{hwg+24} on Page 9]
\label{lem:contextual_map_self_attn_new}

    If the following conditions hold:
    \begin{itemize}
        \item Let $X^{(1)}, \hdots, X^{(N)} \in \R^{n \times d}$ be embeddings that are $(\gamma_{\min}, \gamma_{\max}, \epsilon)$-tokenwise separated, with the vocabulary set $\mathcal{V} = \cup_{i \in [N]} \mathcal{V}^{(i)} \subset \R^{d}$.
        \item $X_k^{(i)} \neq X_l^{(i)}$ for any $i \in [N]$ and $k, l \in [L]$.
        \item Let $
        \gamma = \gamma_{\max} + \frac{\epsilon}{4}$
        \item Let $
        \delta = \exp(-5 \epsilon^{-1} |{\cal V}|^4 d \kappa \gamma_{\max} \log L )$
        \item Let $\kappa := \gamma_{\max}/\gamma_{\min}$.
        \item Let $W^{(O)} \in \R^{d \times s}$ and $W_V, W_K, W_Q \in \R^{s \times d}$.
    \end{itemize}

    Then, we can show
    \begin{itemize}
        \item 1-layer, single-head attention mechanism serves as a $(\gamma, \delta)$-contextual mapping for the embeddings $X^{(1)}, \hdots, X^{(N)}$ with weight matrices $W^{(O)}$ and $W_V, W_K, W_Q$. 
    \end{itemize}
    
\end{lemma}

Lemma~\ref{lem:contextual_map_self_attn_new} indicates that any-rank self-attention function distinguishes input tokens $X_k^{(i)}=X_l^{(j)}$ such that $\mathcal{V}^{(i)} \neq \mathcal{V}^{(j)}$.
In other words, 
it distinguishes two identical tokens within a different context.

%% file: 5_var.tex
\section{Universality of VAR Transformer}\label{sec:var_mainresult}

In this section, we present our proof for the universality of the VAR Transformer.
In Section~\ref{sec:universality_2}, we used a universality result from a previous work.
In Section~\ref{sec:two_layer_pertu}, we analyze how the error behaves when two consecutive layers in our composition are each replaced by their respective approximations.
In Section~\ref{sec:pertu_recursive_one}, we present the scenario when one of the composited layers got replaced by a different function.
In Section~\ref{sec:pertu_recursive_all}, we present the scenario when all of the composited layers got replaced. 
In Section~\ref{sec:var_universality}, we present our proof for the universality of the VAR Transformer.

\subsection{Universality of \texorpdfstring{$\mathcal{T}_A^{1,1,4}$}{} with \texorpdfstring{$O((1/ \epsilon)^{d 
n})$}{} FFN Layers}
\label{sec:universality_2}

We used a universality result from~\cite{hwg+24}.

\begin{lemma}[$\tau \in \mathcal{T}^{1,1,4}_A$ Transformer is Universal Seq2Seq Approximator, Theorem 2.3 in~\cite{hwg+24} on Page 11]
\label{lem:PT_uni_multi_layer_FF2}
    If the following conditions hold: 
    \begin{itemize}
        \item Let $1 \leq p < \infty$ and $\epsilon > 0$.
        \item Let a transformer with one self-attention layer defined as $\tau \in \mathcal{T}^{1,1,4}_A$
    \end{itemize}

    Then, there exists
    \begin{itemize}
        \item a transformer $\tau$ with single self-attention layer, such that for any $\mathcal{L} \in \mathcal{F}_{C}$ there exists
    $ \| \tau ( \cdot), \mathcal{L}\|_\alpha \leq \epsilon$.
    \end{itemize}
\end{lemma}

\subsection{Two Layers Perturbation}\label{sec:two_layer_pertu}

In this section, we analyze how the error behaves when two consecutive layers in our composition are each replaced by their respective approximations. Specifically, we consider the composition $f_i \circ g_i$ and replace $g_i$ with an up interpolation function $\Phi_{\mathrm{up},i}$ and $f_i$ with a one-layer transformer $\tau_i$. We show that under appropriate Lipschitz and approximation assumptions, the overall error of the approximated two-layer composition can be controlled in terms of the individual approximation errors.

\begin{assumption}[Target Function Class]\label{as:function_class}
    We assume the following things:
    \begin{itemize}
        \item Let $f_1, \ldots, f_r$ be $r$ $K$-Lipschitz functions from $\R^{h_r \times w_r \times d}$ to $\R^{h_r \times w_r \times d}$.
        \item For each $i \in [r]$, let $g_i$ be a $K$-Lipschitz function from $\R^{h_{i-1} \times w_{i-1} \times d}$ to $\R^{h_i \times w_i \times d}$.
        \item We assume that for each $i \in [r]$, $g_i$ can be approximated by some up interpolation function $\phi_{\mathrm{up},i}$.
        \item We assume that the target function $f_{\mathrm{word2img}}:\R^{1 \times 1 \times d} \to \R^{h_r \times w_r \times d}$ satisfies
    \begin{align*}
        f_{\mathrm{word2img}} := f_r \circ g_r \cdots \circ f_1 \circ g_1.
    \end{align*}
    \end{itemize}
\end{assumption}

With the Assumption~\ref{as:function_class}, we present the two layers of perturbation as follows. 

\begin{lemma}[Two Layers Perturbation]
    Let $\phi_{{\rm up},i}$ be the up interpolation function defined in~\ref{def:up_inter_layer_one_step}. Let $f_r$ be $r$ $K$-Lipschitz functions from Assumption~\ref{as:function_class}. Let $g_i$ be $r$ $K$-Lipschitz functions from Assumption~\ref{as:function_class}. Let $\tau_i$ be the one-layer transformer defined in Eq. 2.4 from~\cite{hwg+24}. If the following conditions hold:
    \begin{itemize}
        \item $\|g_i - \Phi_{\mathrm{up},i}\| \leq \epsilon_{1,i}$ from Assumption~\ref{as:function_class}.
        \item $\|f_i - \tau_i\| \leq \epsilon_{2,i}$ from Theorem~\ref{lem:PT_uni_multi_layer_FF2}.
        \item $f_i$ is $K_{1,i}$-Lipschitz.
    \end{itemize}
    Then we have
    \begin{align*}
        \|f_i \circ g_i - \tau_i \circ \Phi_{\mathrm{up},i}\| \leq K_{1,i} \epsilon_{1,i} + \epsilon_{2,i}.
    \end{align*}
\end{lemma}
\begin{proof}
    We can show that
    \begin{align*}
         \| f_i \circ g_i -  \tau_i \circ \Phi_{\mathrm{up},i} \| 
         = &~ \|f_i \circ g_i- f_i \circ \Phi_{\mathrm{up},i} + f_i \circ \Phi_{\mathrm{up},i}- \tau_i \circ \Phi_{\mathrm{up},i}\| \\
         \leq &~ \|f_i \circ g_i- f_i \circ \Phi_{\mathrm{up},i}\| + \|f_i \circ \Phi_{\mathrm{up},i}- \tau_i \circ \Phi_{\mathrm{up},i}\| \\
         = &~ \|f_i \circ (g_i- \Phi_{\mathrm{up},i})\| + \|(f_i -\tau_i) \circ \Phi_{\mathrm{up},i}\| \\
         \leq &~ \|f_i \circ (g_i- \Phi_{\mathrm{up},i})\| + \|f_i -\tau_i \| \\
         \leq &~ K_{1,i} \epsilon_{1,i} + \epsilon_{2,i}
    \end{align*}
    where the first step follows from basic algebra, the second step follows from triangle inequality, the third and fourth steps follow from basic algebra, and the fifth step follows from our conditions.
\end{proof}

\subsection{Perturbation of Recursively Composting Functions that One Layer is Different}\label{sec:pertu_recursive_one}

In this section, we consider a scenario where we have a composition of many layers, but only one of the layers is replaced by a different function. This setting helps us see how a single local perturbation can propagate through subsequent layers in a multi-layer composition. The lemma below quantifies this propagation by leveraging Lipschitz continuity.

\begin{lemma}[Perturbation of Recursively Composting Functions, One Layer is Different]\label{lem:one_layer_perturbation}
if the following conditions hold
\begin{itemize}
    \item Assume $\|  u_j(w) - v_j(w) \| \leq \epsilon $ for any $w$.
    \item $v_i(x) \leq K_2 \cdot \| x \|$
\end{itemize}
    Fix $j$, we have 
    \begin{align*}
        \| \circ_{i=j+1}^{n+1} v_i  \circ_{i=1}^j u_i - \circ_{i=j}^n v_i  \circ_{i=0}^{j-1} u_i \| \leq K_2^{n-j} \cdot \epsilon
    \end{align*}
\end{lemma}

\begin{proof}
We define $u$
\begin{align*}
    w = \circ_{i=0}^{j-1} u_i (x)
\end{align*}

    We can show that for any $x$
    \begin{align*}
        \| \circ_{i=j+1}^{n+1} v_i  \circ_{i=1}^j u_i (x) - \circ_{i=j}^n v_i  \circ_{i=0}^{j-1} u_i (x) \| 
        = & ~  \| \circ_{i=j+1}^{n+1} v_i  u_j(w) - \circ_{i=j+1}^n v_i  ( v_j(w) ) \| \\
        = & ~ \| \circ_{i=j+1}^{n+1} v_i ( u_j(w) - v_j(w)  ) \| \\
        \leq & ~ K_2^{n-j} \cdot \epsilon
    \end{align*}
    where the first step follows from basic algebra, the second step follows from linearity, and the third step follows from lemma assumptions.
\end{proof}

\subsection{Perturbation of Recursively Composting Functions that All Layer are Different}\label{sec:pertu_recursive_all}

In this section, we extend the analysis to the most general scenario in which all layers in the composition are replaced by different functions. This captures the situation where each layer $u_i$ is approximated by some other function $v_i$. We derive a cumulative bound that sums the individual perturbations introduced at each layer.

\begin{lemma}[Perturbation of Recursively Compositing Functions, All Layers are Different]
    If the following conditions hold:
    \begin{itemize}
        \item Let $\circ_{i=1}^n u_i  = u_n \circ \cdots \circ u_1$
        \item Let $\circ_{i=1}$
        \item Let $u_0(x) = x$ which is identity mapping
        \item Let $v_{n+1}(x) = x$ which is identity mapping
    \end{itemize}
    Then 
    \begin{align*}
        \| \circ_{i=1}^n u_i - \circ_{i=1}^n v_i \| \leq\sum_{j=1}^{n} \| \circ_{i=j+1}^{n+1} v_i  \circ_{i=1}^j u_i - \circ_{i=j}^n v_i  \circ_{i=0}^{j-1} u_i \|
    \end{align*}
\end{lemma}
\begin{proof}
    We can show
    \begin{align*}
        \| \circ_{i=1}^n u_i - \circ_{i=1}^n v_i \| 
        = & ~ \| \sum_{j=1}^{n} ( \circ_{i=j+1}^{n+1} v_i  \circ_{i=1}^j u_i - \circ_{i=j}^n v_i  \circ_{i=0}^{j-1} u_i ) \| \\
        \leq & ~  
        \sum_{j=1}^{n} \| \circ_{i=j+1}^{n+1} v_i  \circ_{i=1}^j u_i - \circ_{i=j}^n v_i  \circ_{i=0}^{j-1} u_i \|
    \end{align*}
    where the first step follows from adding intermediate terms, and the last step follows from the triangle inequality.
    
Thus, we complete the proof.
\end{proof}

\subsection{The Universality of VAR Transformer}\label{sec:var_universality}

In this section, with the established error bounds for replacing individual or multiple layers with alternative functions, we now prove the main universality result for the VAR Transformer. In essence, we show that a properly constructed VAR Transformer can approximate the target function $f_\mathrm{word2img}$ (from Assumption~\ref{as:function_class}) with arbitrarily small errors under suitable Lipschitz and approximation assumptions on each layer.

\begin{theorem}[Universality of VAR Transformer]\label{thm:var_universality}
    Assume $K_2 > 2$.
    For $f_\mathrm{word2img}$ satisfies Assumption~\ref{as:function_class}, there exists a VAR Transformer $\tau_{\mathsf{VAR}}$ such that
    \begin{align*}
        \|\tau_{\mathsf{VAR}} - f_\mathrm{word2img} \| \leq K_2^n( K_{1,i} \epsilon_{1,i} + \epsilon_{2,i} ).
    \end{align*}
\end{theorem}
\begin{proof}
    We can show that
    \begin{align*}
        \|\tau_{\mathsf{VAR}} - f_\mathrm{word2img} \|  = &~ \| \circ_{i=1}^r (f_i \circ g_i) - \circ_{i=1}^r (\tau_i \circ \Phi_{\mathrm{up},i}) \| \\
        = &~ \sum_{j=1}^n K_2^{n-j} ( K_{1,i} \epsilon_{1,i} + \epsilon_{2,i} )\\
        = &~ \frac{K_2^n - 1}{K_2 - 1}( K_{1,i} \epsilon_{1,i} + \epsilon_{2,i} ) \\
        \leq &~ K_2^n( K_{1,i} \epsilon_{1,i} + \epsilon_{2,i} ) .
    \end{align*}
    where the first step follows from Definition~\ref{def:var_function_class}, the second step follows from Lemma~\ref{lem:one_layer_perturbation}, the third step follows from basic algebra, and the fourth step follows from the basic inequality.  
\end{proof}

%% file: 6_flowar.tex
\section{Universality of FlowAR}\label{sec:flowar_universality}

In this section, we show that the universality results established for the VAR Transformer can be extended to our FlowAR model. The key observation is that the same local perturbation bounds and Lipschitz assumptions used in the VAR Transformer setting also apply to FlowAR, with only minor changes. Specifically, each FlowAR layer $\Phi_{\mathrm{down}, i}$ can be analyzed in an analogous way to $\Phi_{\mathrm{up}, i}$, allowing us to derive a bound on the overall error of the composed FlowAR model.

\begin{corollary}
    Let $\phi_{{\rm down},i}$ be the down interpolation function of FlowAR (see Definition~\ref{def:down_sample_function}). Let $f_r$ be $r$ $K$-Lipschitz functions from Assumption~\ref{as:function_class}. Let $g_i$ be $r$ $K$-Lipschitz functions from Assumption~\ref{as:function_class}. Let $\tau_i$ be the one-layer transformer defined in Eq. 2.4 from~\cite{hwg+24}. 
    If the following conditions hold: 
    \begin{itemize}
        \item $\|g_i - \Phi_{\mathrm{down},i}\| \leq \epsilon_{1,i}$
        \item $\|f_i - \tau_i\| \leq \epsilon_{2,i}$
        \item $f_i$ is $K_{1,i}$-Lipschitz
        
    \end{itemize}
    Then we have
    \begin{align*}
        \|f_i \circ g_i - \tau_i \circ \Phi_{\mathrm{down},i}\| \leq K_{1,i} \epsilon_{1,i} + \epsilon_{2,i}.
    \end{align*}
\end{corollary}

The proof of this corollary mirrors the two-layer perturbation argument from the VAR Transformer, except each ``up'' interpolation function $\Phi_{\mathrm{up},i}$ is replaced by the corresponding ``down'' interpolation function $\Phi_{\mathrm{down},i}$. The same Lipschitz and approximation assumptions allow us to bound the difference between $f_i \circ g_i$ and $\tau_i \circ \Phi_{\mathrm{down},i}$.

\begin{corollary}\label{corollary:flowar_universality}
    There exists a FlowAR model such that
    \begin{align*}
        \| \tau_{\mathsf{FlowAR}} - f_{\mathrm{word2img}} \| \leq O(\epsilon).
    \end{align*}
\end{corollary}

The proof of Corollary~\ref{corollary:flowar_universality} follows the same high-level structure as our universality results for the VAR Transformer. By applying the local perturbation bound layer by layer and then summing the resulting errors, we obtain a global approximation guarantee that is $O(\epsilon)$. Hence, FlowAR, just like the VAR Transformer, can universally approximate the target function $f_{\mathrm{word2img}}$ under the given Lipschitz and approximation assumptions.

%% file: 7_conclusion.tex
\section{Conclusion}\label{sec:conclusion}

In this paper, we established that both VAR Transformers and FlowAR architectures serve as universal approximators for Lipschitz sequence-to-sequence functions—even in their most minimal configurations. By dissecting the roles of self-attention, multi-scale up-sampling, and invertible flow transformations, we showed how these components collectively endow the models with sufficient expressive power to capture arbitrary continuous mappings. Our results unify previous theoretical findings on transformer universality with the practical enhancements brought by VAR and flow-based designs, providing a deeper theoretical underpinning for their empirical successes in high-quality image generation and structured prediction tasks. We hope that these findings inspire new explorations into more advanced architectural variants and guide future work on balancing model efficiency, interpretability, and expressive power.

%% file: 8_app_prelim.tex
{\bf Roadmap} In Section~\ref{sec:app_prelim}, we provide basic algebras that support our proofs. 
In Section~\ref{sec:app_var}, we provide other phases of the VAR Model. 
In Section~\ref{sec:training_of_flowar}, we give the definitions for training the FlowAR Model. 
In Section~\ref{sec:inference_of_flowar}, we give the definitions for the inference of the FlowAR Model.
In Section~\ref{sec:more_work}, we introduce more related work.

\section{Preliminary}\label{sec:app_prelim}

In this section, we introduce notations and basic facts that are used in our work. We first list some basic facts of matrix norm properties.

\subsection{Notations}
We denote the $\ell_p$ norm of a vector $x$ by $\| x \|_p$, i.e., $\|x\|_1 := \sum_{i=1}^n |x_i|$, $\| x \|_2 := (\sum_{i=1}^n x_i^2)^{1/2}$ and $\| x \|_{\infty} := \max_{i \in [n]} |x_i|$. For a vector $x \in \R^n$, $\exp(x) \in \R^n$ denotes a vector where $\exp(x)_i$ is $\exp(x_i)$ for all $i \in [n]$. For $n > k$, for any matrix $A \in \R^{n\times k}$, we denote the spectral norm of $A$ by $\| A \|$, i.e., $\| A \| := \sup_{x\in \R^k} \| Ax \|_2 / \| x \|_2$. We define the function norm as $\| f \|_\alpha :=  (\int \| f(X) \|_\alpha^\alpha \d X)^{1/\alpha}$ where $f$ is a function. We use $\sigma_{\min}(A)$ to denote the minimum singular value of $A$. Given two vectors $x, y \in \R^n$, we use $\langle x, y \rangle$ to denote $\sum_{i=1}^n x_iy_i$. Given two vectors $x, y \in \R^n$, we use $x \circ y$ to denote a vector that its $i$-th entry is $x_i y_i$ for all $i \in [n]$. We use $e_i \in \R^n$ to denote a vector where $i$-th entry is $1$, and all other entries are $0$. Let $x \in \R^n$ be a vector. We define $\diag (x) \in \R^{n \times n}$ as the diagonal matrix whose diagonal entries are given by $\diag(x)_{i, i} = x_i$ for $i = 1, \dots, n$, and all off-diagonal entries are zero. For a symmetric matrix $A \in \R^{n\times n}$, we say $A \succ 0$ (positive definite (PD)), if for all $x\in \R^n \setminus \{ {\bf 0}_n \}$, we have $x^\top A x > 0$. For a symmetric matrix $A \in \R^{n \times n}$, we say $A \succeq 0$ (positive semidefinite (PSD)), if for all $x \in \R^n$, we have $x^\top A x \geq 0$. The Taylor Series for $\exp(x)$ is $\exp(x) = \sum_{i=0}^{\infty} \frac{x^i}{i!}$. For a matrix $X \in \R^{n_1 n_2 \times d}$, we use $\X \in \R^{n_1 \times n_2 \times d}$ to denote its tensorization, and we only assume this for letters $X$ and $Y$.

\subsection{Basic Algebra}
In this section, we introduce the basic algebras used in our work.

\begin{fact}
Let $A$ denote the matrix. For each $i$, we use $A_{i,*}$ to denote the $i$-th row of $A$. For $j$, we use $A_{*,j}$ to denote the $j$-th column of $A$. 
We can show that
\begin{itemize}
    \item $\| A \| \leq \| A \|_F$
    \item $\| A \| \geq \| A_{i,*} \|_2$
    \item $\| A \| \geq \| A_{*,j} \|_2$
\end{itemize}
\end{fact}

Then, we introduce some useful inner product properties.

\begin{fact}
    For vectors $u,v,w \in \R^n$. We have 
    \begin{itemize}
        \item $\langle u, v \rangle = \langle u \circ v, {\bf 1}_n \rangle$
        \item $\langle u \circ v, w \rangle = \langle u \circ v \circ w, {\bf 1}_n \rangle$ 
        \item $\langle u, v \rangle = \langle v, u \rangle$
        \item $\langle u , v \rangle = u^\top v = v^\top u$
    \end{itemize}
\end{fact}

Now, we show more vector properties related to the hadamard products, inner products, and diagnoal matrices.
\begin{fact}
    For any vectors $u, v, w \in \R^n$, we have
    \begin{itemize}
        \item $u \circ v = v \circ u = \diag(u) \cdot v = \diag(v) \cdot u$
        \item $u^\top (v\circ w) = u^\top \diag(v) w$
        \item $u^\top (v\circ w) = v^\top (u \circ w) = w^\top (u \circ v)$
        \item $u^\top \diag(v) w = v^\top \diag(u) w = u^\top \diag(w) v$
        \item $\diag(u) \cdot \diag(v) \cdot {\bf 1}_n = \diag(u) v$
        \item $\diag(u \circ v) = \diag(u) \diag(v)$
        \item $\diag(u) + \diag(v) = \diag(u+v)$
    \end{itemize}
\end{fact}

%% file: 9_app_var.tex
\section{VAR Transformer Blocks}\label{sec:app_var}

In this section, we define the components in the VAR Transformer.

We first introduce the Softmax unit.

\begin{definition}[Softmax] \label{def:softmax}
    Let $z \in \R^{n}$. We define 
    $\mathsf{Softmax}: \R^{n} \to \R^{n}$ satisfying 
    \begin{align*}
        \mathsf{Softmax}(z):= \exp(z) / \langle \exp(z) , {\bf 1}_n \rangle.  
    \end{align*}
\end{definition}

Here, we define the attention matrix in the VAR Transformer as follows.

\begin{definition}[Attention Matrix]\label{def:attn_matrix}
    Let $W_Q, W_K \in \R^{d \times d}$ denote the model weights. Let $X \in \R^{n \times d}$ denote the representation of the length-$n$ input. Then, we define the attention matrix $A \in \R^{n \times n}$ by, For $i,j \in [n]$, 
    \begin{align*}
        A_{i,j} := & ~\exp( \underbrace{ X_{i,*} }_{1 \times d} \underbrace{ W_Q }_{d \times d} \underbrace{ W_K^\top }_{d \times d} \underbrace{ X_{j,*}^\top }_{d \times 1}).
    \end{align*}
\end{definition}

With the attention matrix, we now provide the definition for a single layer of Attention.

\begin{definition}[Single Attention Layer]\label{def:single_layer_transformer}
     Let $X \in \R^{n \times d}$ denote the representation of the length-$n$ sentence. Let $W_V \in \R^{d \times d}$ denote the model weights. As in the usual attention mechanism, the final goal is to output an $n \times d$ size matrix where $D:= \diag( A {\bf 1}_n) \in \R^{n \times n}$. Then, we define attention layer $\mathsf{Attn}$ as
    \begin{align*}
        \mathsf{Attn} (X) := & ~ D^{-1} A X W_V .
    \end{align*}
\end{definition}

Here we present the definition of the VAR Attention.

\begin{definition}[$\VAR$ Attention Layer]\label{def:single_layer_var_transformer}
    Let $r \geq 1$ be a positive integer. Let $h_r, w_r$ be two positive integers. 
     Let $\X \in \R^{h_r \times w_r \times d}$ denote the representation of the input token map. Let $W_V \in \R^{d \times d}$ denote the model weights. As in the usual attention mechanism, the final goal is to output an $n \times d$ size matrix where $D:= \diag( A {\bf 1}_n) \in \R^{h_rw_r \times h_rw_r}$. Then, we define attention layer $\mathsf{Attn}_r: \R^{h_rw_r \times d} \to \R^{h_rw_r \times d}$ as
    \begin{align*}
        \mathsf{Attn}_r (X) := & ~ D^{-1} A X W_V .
    \end{align*}
\end{definition}

We introduce the feed-forward layer in the VAR Transformer as follows.

\begin{definition}[Single Feed-Forward Layer]\label{def:ffn}
We define the FFN as follows:
    \begin{itemize}
        \item $X \in \R^{d\times L}$
        \item $k \in [n]$
        \item $c$ is the number of neurons
        \item $W^{(1)} \in \R^{c \times d}, W^{(2)} \in \R^{d \times c}$ are weight matrices
        \item $b^{(1)}\in \R^{ c}$, $ b^{(2)} \in \R^{d}$ are bias vectors.
        \item $\mathsf{FFN}: \R^{d\times L} \to \R^{d\times L}$ 
    \end{itemize}
    
    \begin{align*}
        \mathsf{FFN}(X)_{*,k} =
        & ~ \underbrace{X_{*,k}}_{d \times 1} +  \underbrace{W^{(2)}}_{d \times c} \mathsf{ReLU}( \underbrace{ W^{(1)} X_{*,k} }_{c\times 1} + \underbrace{b^{(1)}}_{c\times 1})  + \underbrace{b^{(2)}}_{d \times 1} 
    \end{align*}
    
\end{definition}

\subsection{Phase 2: Feature Map Reconstruction }\label{sec:phase_2}

In this section, we introduce the Phase Two of the VAR model.

\begin{definition}[Convolution Layer, Definition 3.9 from~\cite{kll+25} on Page 9]\label{def:conv_layer}
    The Convolution Layer is defined as follows:
    \begin{itemize}
        \item Let $h \in \mathbb{N}$ denote the height of the input and output feature map.
        \item Let $w \in \mathbb{N}$ denote the width of the input and output feature map.
        \item Let $c_{\rm in} \in \mathbb{N}$ denote the number of channels of the input feature map.
        \item Let $c_{\rm out} \in \mathbb{N}$ denote the number of channels of the output feature map.
        \item Let $X \in \R^{h \times w \times c_{\rm in}}$ denote the input feature map.
        
        \item For $l \in [c_{\rm out}]$, we use $K^l \in \R^{3 \times 3 \times c_{\rm in}}$ to denote the $l$-th convolution kernel.
        \item Let $p = 1$ denote the padding of the convolution layer.
        \item Let $s = 1$ denote the stride of the convolution kernel.
        \item Let $Y \in \R^{h \times w \times c_{\rm out}}$ denote the output feature map.
    \end{itemize}
    We use $\phi_{\rm conv}: \R^{h \times w \times c_{\rm in}} \to \R^{h \times w \times c_{\rm out}}$ to denote the convolution operation then we have $Y = \phi_{\rm conv}(X)$. Specifically, for $i \in [h], j \in [w], l \in [c_{\rm out}]$, we have
    \begin{align*}
        Y_{i,j,l} := \sum_{m=1}^3 \sum_{n=1}^3 \sum_{c = 1}^{c_{\rm in}} X_{i+m-1,j+n-1,c} \cdot K^l_{m,n,c} + b
    \end{align*}
\end{definition}

\begin{remark}
    Assumptions of kernel size, padding of the convolution layer, and stride of the convolution kernel are based on the specific implementation of \cite{tjy+24}.
\end{remark}

\subsection{Phase 3: VQ-VAE Decoder process}\label{sec:phase_3} 

In this section, we introduce Phase Three of the VAR model.

VAR will use the VQ-VAE Decoder Module to reconstruct the feature map generated in Section~\ref{sec:phase_2} into a new image. The Decoder of VQ-VAE has the following main modules \cite{kll+25}: (1) Resnet Blocks; (2) Attention Blocks; (3) Up Sample Blocks. We recommend readers to \cite{kll+25} for more details.

%% file: 10_app_flowar.tex
\section{Training of FlowAR}\label{sec:training_of_flowar}

In this section, we introduce the training of FlowAR along with its definitions based on~\cite{ryh+24}.

We first introduce some notations used in FlowAR.
\subsection{Notations}\label{sub:notations}
For a matrix $X \in \R^{n_1 n_2 \times d}$, we use $\X \in \R^{n_1 \times n_2 \times d}$ to denote its tensorization, and we only assume this for letters $X, Y, Z, F, V$.

\subsection{Sample Function}
In this section, we introduce the sample functions used in FlowAR.

We first introduce the up sample function.

\begin{definition}[Up Sample Function]\label{def:up_sample_function}
    If the following conditions hold:
    \begin{itemize}
        \item Let $h, w \in \mathbb{N}$ denote the height and weight of latent $\X \in \R^{h \times w \times c}$.
        \item Let $r > 0$ denote a positive integer.
    \end{itemize}
    Then we define $\mathrm{Up}(\X,r) \in \R^{rh \times rw \times c}$ as the upsampling of latent $\X$ by a factor $r$.
\end{definition}

Then, we introduce the down sample function.
\begin{definition}[Down Sample Function]\label{def:down_sample_function}
    If the following conditions hold:
    \begin{itemize}
        \item  Let $h, w \in \mathbb{N}$ denote the height and weight of latent $\X  \in \R^{h \times w \times c}$.
        \item Let $r > 0$ denote a positive integer.
    \end{itemize}
    Then we define $\mathrm{Down}(\X ,r) \in \R^{\frac{h}{r} \times \frac{w}{r} \times c}$ as the downsampling of latent $\X $ by a factor $r$.
\end{definition}

\subsection{Linear Sample Function}
In this section, we present the linear sample function in FlowAR.

We first present the linear up sample function.
\begin{definition}[Linear Up Sample Function]\label{def:linear_up_sample_function}
    If the following conditions hold:
    \begin{itemize}
        \item Let $h, w \in \mathbb{N}$ denote the height and weight of latent $\X \in \R^{h \times w \times c}$. 
        \item Let $r > 0$ denote a positive integer.
        \item Let $\Phi_{\mathrm{up}} \in \R^{hw \times (rh \cdot rw)}$ be a matrix. 
    \end{itemize}
    Then we define the linear up sample function $\phi_{\mathrm{up}}(\cdot, \cdot)$ as it computes $\Y := \phi_{\mathrm{up}}(\X,r) \in \R^{rh \times rw \times c}$ such that the matrix version of $\X$ and $\Y$ satisfies
    \begin{align*}
        Y = \Phi_{\mathrm{up}}X \in \R^{(rh\cdot rw) \times c}.
    \end{align*}
\end{definition}

Then, we define linear down sample function as follows.
\begin{definition}[Linear Down Sample Function]\label{def:linear_down_sample_function}
    If the following conditions hold:
    \begin{itemize}
        \item Let $h, w \in \mathbb{N}$ denote the height and weight of latent $\X \in \R^{h \times w \times c}$.
        \item Let $r > 0$ denote a positive integer.
        \item Let $\Phi_{\mathrm{down}} \in \R^{((h/r) \cdot (w/r)) \times hw}$ be a matrix. 
    \end{itemize}
    Then we define the linear down sample function $\phi_{\mathrm{down}}(\X,r)$ as it computes $\Y := \phi_{\mathrm{down}}(\X,r) \in \R^{(h/r) \times (w/r) \times c}$ such that the matrix version of $\X$ and $\Y$ satisfies
    \begin{align*}
        Y = \Phi_{\mathrm{down}}X \in \R^{((h/r) \cdot (w/r)) \times c}.
    \end{align*}
\end{definition}

\subsection{VAE Tokenizer}\label{sub:vae_tokenizer}
In this section, we show the VAE Tokenizer.

\begin{definition}[VAE Tokenizer]\label{def:vae_tokenizer}
If the following conditions hold:
\begin{itemize}
    \item Let $\X \in \R^{h \times w \times c}$ denote a continuous latent representation generated by VAE.
    \item Let $K$ denote the total number of scales in FlowAR.
    \item Let $a$ be a positive integer.
    \item For $i \in [K]$, let $r_i := a^{K-i}$.
    \item For each $i\in [K]$, let $\phi_{\mathrm{down}, i}(\cdot , r_i) : \R^{h \times w \times c} \to \R^{(h / r_i) \times (w/r_i) \times c}$ denote the linear down sample function defined in Definition~\ref{def:linear_down_sample_function}.
\end{itemize}
For $i \in [K]$, we define the $i$-th token map generated by VAE Tokenizer be
\begin{align*}
    \Y^{i} := \phi_{\mathrm{down}, i}(\X, r_i) \in \R^{(h / r_i) \times (w/r_i) \times c},
\end{align*}
We define the output of VAE Tokenizer as follows:
    \begin{align*}
        \mathsf{Tokenizer}(\mathsf{X}) := \{\Y^{1},\Y^{2}, \dots,\Y^{K}\}.
    \end{align*}
\end{definition}
\begin{remark}
    In \cite{ryh+24}, they choose $a =2$ and hence for $i \in [n]$, $r_i := 2^{K-i}$.
\end{remark}

\subsection{Autoregressive Transformer}

Firstly, we give the definition of a single attention layer.
\begin{definition}[Single Attention Layer]\label{def:attn_layer}
    If the following conditions hold:
    \begin{itemize}
        \item Let $h,w \in \mathbb{N}$ denote the height and weight of latent $\X \in \R^{h \times w \times c}$.
        \item Let $W_Q, W_K, W_V \in \R^{c \times c}$ denote the weight matrix for query, key, and value, respectively.
    \end{itemize}
    Then we define the attention layer $\mathsf{Attn}(\cdot)$ as it computes  $\Y = \mathsf{Attn}(\X) \in \R^{h \times w \times c}$. For the matrix version, we first need to compute the attention matrix $A \in \R^{hw \times hw}$:
    \begin{align*}
        A_{i,j} := & ~\exp(  X_{i,*}   W_Q   W_K^\top   X_{j,*}^\top), \text{~~for~} i, j \in [hw].
    \end{align*}
    Then, we compute the output:
    \begin{align*}
        Y := D^{-1}AXW_V \in \R^{hw \times c}.
    \end{align*}
    where $D:=\diag(A {\bf 1}_n) \in \R^{hw \times hw}$.
\end{definition}

To move on, we present the definition of multilayer perceptron.
\begin{definition}[MLP layer]\label{def:mlp}
    If the following conditions hold:
    \begin{itemize}
        \item Let $h,w \in \mathbb{N}$ denote the height and weight of latent $\X \in \R^{h \times w \times c}$.
        \item Let $c$ denote the input dimension of latent $\X \in \R^{h \times w \times c}$.
        \item Let $d$ denote the dimension of the target output.
        \item Let $W \in \R^{c \times d}$ denote a weight matrix.
        \item Let $b \in \R^{1 \times d}$ denote a bias vector.
    \end{itemize}
    Then we define mlp layer as it computes $\Y := \mathsf{MLP}(X,c,d) \in \R^{h \times w \times d}$ such that the matrix version of $\X$ and $\Y$ astisfies, for each $j \in [hw]$,
    \begin{align*}
        Y_{j,:} = \underbrace{X_{j,:}}_{1\times c} \cdot \underbrace{W}_{c \times d} + \underbrace{b}_{1 \times d}
    \end{align*}
\end{definition}

We present the definition of layer-wise norm layer.
\begin{definition}[Layer-wise norm layer]\label{def:ln}
    Given a latent $\X \in \R^{h \times w \times c}$. We define the layer-wise as it computes $\Y := \mathsf{LN}(X) \in \R^{h\times w \times c}$ such that the matrix version of $\X$ and $\Y$ satisfies, for each $j \in [hw]$,
    \begin{align*}
        Y_{j,:} =  \frac{X_{j,:}-\mu_j}{\sqrt{\sigma_j^2}}
    \end{align*}
    where $\mu_j := \sum_{k=1}^c X_{j,k}/c$ and $\sigma_{j}^2 = \sum_{k=1}^c(X_{j,k}-\mu_j)^2/c$.
\end{definition}

\begin{definition}[Autoregressive Transformer]\label{def:ar_transformer}
    If the following conditions hold:
    \begin{itemize}
        \item Let $\X \in \R^{h \times w \times c}$ denote a continuous latent representation generated by VAE.
        \item Let $K$ denote the total number of scales in FlowAR.
        \item For $i \in [K]$, let $\Y_i \in \R^{(h / r_i) \times (w/r_i) \times c}$ be the $i$-th token map genereated by VAE Tokenizer defined in Definition~\ref{def:vae_tokenizer}.
        \item Let $a$ be a positive integer.
        \item For $i \in [K]$, let $r_i := a^{K-i}$.
        \item For $i \in [K-1]$, let $\phi_{\mathrm{up}, i}(\cdot, a):\R^{(h/r_i)\times(w/r_i)\times c} \to \R^{(h/r_{i+1})\times(w/r_{i+1})\times c}$ be the linear up sample function defined in Definition~\ref{def:linear_up_sample_function}.
        \item For $i \in [K]$, let $\mathsf{Attn}_i(\cdot):\R^{(\sum_{j=1}^i h/r_j)(\sum_{j=1}^i w/r_j) \times c} \to \R^{(\sum_{j=1}^i h/r_j)(\sum_{j=1}^i w/r_j) \times c}$ be the $i$-th attention layer defined in Definition~\ref{def:attn_layer}.
        \item For $i \in [K]$, let $\mathsf{FFN}_i(\cdot):\R^{(\sum_{j=1}^i h/r_j)(\sum_{j=1}^i w/r_j) \times c} \to \R^{(\sum_{j=1}^i h/r_j)(\sum_{j=1}^i w/r_j) \times c}$ be the $i$-th feed forward network defined in Definition~\ref{def:ffn}.
        \item Let $\Z_{\mathrm{init}} \in \R^{(h/r_1) \times (w/r_1) \times c}$ be the initial input denoting the class condition.
        \item Let $\Z^1 : = \mathsf{Z}_{\mathrm{init}} \in \R^{(h/r_1) \times (w/r_1) \times c}$.
        \item For $i \in [K] \setminus \{1\}$, Let $\Z^{i}$ be the reshape of the input sequence $\mathsf{Z}_{\mathrm{init}}, \phi_{\mathrm{up}, 1}(\Y^1, a), \ldots, \phi_{\mathrm{up}, i}(\Y^{i-1}, a)$ into the tensor of size $(\sum_{j=1}^i h /r_{j}) \times (\sum_{j=1}^i w /r_{j}) \times c$.
    \end{itemize}
    For $i \in [K]$, we define the Autoregressive transformer $\mathsf{TF}_i$ as 
    \begin{align*}
        \mathsf{TF}_i(\Z^i) = \mathsf{FFN}_i \circ \mathsf{Attn}_i (\Z^i) \in \R^{(\sum_{j=1}^{i} h /r_{j})(\sum_{j=1}^{i} w /r_{j}) \times c}.
    \end{align*}
    We denote $\wh{\Y}^i$ as the $i$-th block of size $(h/r_{i}) \times (w/r_{i}) \times c$ of the tensorization of $\mathsf{TF}_i(Z^i)$.
\end{definition}

\subsection{Flow Matching}
In this section, we introduce the flow matching definition.

\begin{definition}[Flow]\label{def:flow}
    If the following conditions hold:
    \begin{itemize}
        \item Let $\X \in \R^{h \times w \times c}$ denote a continuous latent representation generated by VAE.
        \item Let $K$ denote the total number of scales in FlowAR.
        \item For $i \in [K]$, let $\Y_i \in \R^{(h / r_i) \times (w/r_i) \times c}$ be the $i$-th token map genereated by VAE Tokenizer defined in Definition~\ref{def:vae_tokenizer}.
        \item For $i \in [K]$, let $\F^i_0 \in \R^{(h / r_i) \times (w/r_i) \times c}$ be a matrix where each entry is sampled from the standard Gaussian $\N(0,1)$.
    \end{itemize}
    We defined the interpolated input as follows:
    \begin{align*}
        \F^i_{t} := t \Y^i + (1-t)\F^i_0.
    \end{align*}
    The velocity flow is defined as
    \begin{align*}
        \V^i_t := \frac{\d \F^i_{t}}{\d t} = \Y^i -\F^i_0.
    \end{align*}
\end{definition}

Here, we define the architecture of the flow matching model defined in \cite{ryh+24}.
\begin{definition}[Flow Matching Architecture]\label{def:flow_matching_architecture}
    If the following conditions hold:
    \begin{itemize}

        \item Let $\X \in \R^{h \times w \times c}$ denote a continuous latent representation generated by VAE.
        \item Let $K$ denote the total number of scales in FlowAR.
        \item For $i \in [K]$, let $\Y_i \in \R^{(h / r_i) \times (w/r_i) \times c}$ be the $i$-th token map genereated by VAE Tokenizer defined in Definition~\ref{def:vae_tokenizer}.
        \item Let $i \in [K]$.
        \item Let $\wh{\Y}_i \in \R^{(h / r_i) \times (w/r_i) \times c}$ be the $i$-th block of the output of Autoregressive Transformer defined in Definition~\ref{def:ar_transformer}.
        \item Let $\F^i_t$ be the interpolated input defined in Definition~\ref{def:flow}.
        \item Let $\mathsf{Attn}_i(\cdot):\R^{(h / r_i) \times (w/r_i) \times c} \to \R^{(h / r_i) \times (w/r_i) \times c}$ be the $i$-th attention layer defined in Definition~\ref{def:attn_layer}.
        \item Let $\mathsf{MLP}_i(\cdot,c,d):\R^{(h / r_i) \times (w/r_i) \times c}  \to \R^{(h / r_i) \times (w/r_i) \times d}$ be the $i$-th attention layer defined in Definition~\ref{def:mlp}.
        \item Let $\mathsf{LN}_i(\cdot): \R^{(h / r_i) \times (w/r_i) \times c}  \to \R^{(h / r_i) \times (w/r_i) \times c}$ be the $i$-th layer-wise norm layer defined in Definition~\ref{def:ln}.
        \item Let $t_i \in [0,1]$ denote a time step.
    \end{itemize}
    Then we define the $i$-th flow matching model as $\mathsf{NN}_i(\F_t^i,\wh{\Y}_i,t_i): \R^{(h / r_i) \times (w/r_i) \times c} \times \R^{(h / r_i) \times (w/r_i) \times c} \times \R \to \R^{(h / r_i) \times (w/r_i) \times c}$. The tensor input needs to go through the following computational steps:
    \begin{itemize}
        \item {\bf Step 1:} Compute intermediate variables $\alpha_1, \alpha_2, \beta_1, \beta_2, \gamma_1, \gamma_2$. Specifically, we have
        \begin{align*}
            \alpha_1, \alpha_2, \beta_1, \beta_2, \gamma_1, \gamma_2 :=&~ \mathsf{MLP}_i(\wh{\Y}_i + t_i \cdot {\bf 1}_{(h / r_i) \times (w/r_i) \times c},c,6c)
        \end{align*}
        \item {\bf Step 2:} Compute intermediate variable $\wh{F_t}^{i'}$. Specifically, we have
        \begin{align*}
            \wh{\F_t}^{i'}:= \mathsf{Attn}_i (\gamma_1 \circ \mathsf{LN}(\F_t^i) + \beta_1) \circ \alpha_1
        \end{align*}
        where $\circ$ denotes the element-wise product for tensors.

        \item {\bf Step 3:} Compute final output $\F_t^{i''}$. Specifically, we have
        \begin{align*}
            \F_t^{i''} = \mathsf{MLP}_i(\gamma_2 \circ \mathsf{LN}(\wh{F_t}^{i'})+ \beta_2,c,c) \circ \alpha_2
        \end{align*}
        where $\circ$ denotes the element-wise product for tensors.
    \end{itemize}
\end{definition}

Then, we present our training objective.
\begin{definition}[Loss of FlowAR]
If the following conditions hold:
    \begin{itemize}
        \item Let $\X \in \R^{h \times w \times c}$ denote a continuous latent representation generated by VAE.
        \item Let $K$ denote the total number of scales in FlowAR.
        \item For $i \in [K]$, let $\Y_i \in \R^{(h / r_i) \times (w/r_i) \times c}$ be the $i$-th token map genereated by VAE Tokenizer defined in Definition~\ref{def:vae_tokenizer}.
        \item For $i \in [K]$, let $\wh{\Y}_i \in \R^{(h / r_i) \times (w/r_i) \times c}$ be the $i$-th block of the output of Autoregressive Transformer defined in Definition~\ref{def:ar_transformer}.
        \item For $i \in [K]$, let $\F^i_t$ be the interpolated input defined in Definition~\ref{def:flow}.
        \item For $i \in [K]$, let $\V^i_t$ be the velocity flow defined in Definition~\ref{def:flow}.
        \item For $i \in [K]$, let $\mathsf{NN}_i(\cdot,\cdot,\cdot):\R^{(h / r_i) \times (w/r_i) \times c} \times \R^{(h / r_i) \times (w/r_i) \times c} \times \R \to \R^{(h / r_i) \times (w/r_i) \times c}$ denote the $i$-th flow matching network defined in Definition~\ref{def:flow_matching_architecture}.
    \end{itemize}
    The loss function of FlowAR is 
    \begin{align*}
        L(\theta) = \sum_{i=1}^n \E_{t \sim \mathsf{Unif}[0,1]}\|\mathsf{NN}_i(\F^i_t, \wh{\Y}^{i}_t, t_i) - \V^i_t\|^2.
    \end{align*}
\end{definition}

\section{Inference of FlowAR}\label{sec:inference_of_flowar}
We define the architecture of FlowAR during the inference process as follows.
\begin{definition}[FlowAR Architecture in the Inference Pipeline]\label{def:flow_architecture_inference}
    If the following conditions hold:
    \begin{itemize}
        \item Let $K$ denote the total number of scales in FlowAR.
        \item Let $a$ be a positive integer.
        \item For $i \in [K]$, let $r_i := a^{K-i}$.
        \item For $i \in [K-1]$, let $\phi_{\mathrm{up}, i}(\cdot, a):\R^{(h/r_i)\times(w/r_i)\times c} \to \R^{(h/r_{i+1})\times(w/r_{i+1})\times c}$ be the linear up sample function defined in Definition~\ref{def:linear_up_sample_function}.
        \item For $i \in [K]$, let $\mathsf{Attn}_i(\cdot):\R^{(\sum_{j=1}^i h/r_j)(\sum_{j=1}^i w/r_j) \times c} \to \R^{(\sum_{j=1}^i h/r_j)(\sum_{j=1}^i w/r_j) \times c}$ be the $i$-th attention layer defined in Definition~\ref{def:attn_layer}.
        \item For $i \in [K]$, let $\mathsf{FFN}_i(\cdot):\R^{(\sum_{j=1}^i h/r_j)(\sum_{j=1}^i w/r_j) \times c} \to \R^{(\sum_{j=1}^i h/r_j)(\sum_{j=1}^i w/r_j) \times c}$ be the $i$-th feed forward network defined in Definition~\ref{def:ffn}.
        \item For $i \in [K]$, let $\mathsf{NN}_i(\cdot,\cdot,\cdot):\R^{(h / r_i) \times (w/r_i) \times c} \times \R^{(h / r_i) \times (w/r_i) \times c} \times \R \to \R^{(h / r_i) \times (w/r_i) \times c}$ denote the $i$-th flow matching network defined in Definition~\ref{def:flow_matching_architecture}.
        \item For $i \in [K]$, let $t_i \in [0,1]$ denote the time steps.
        \item For $i \in [K]$, let $\F^i_t$ be the interpolated input defined in Definition~\ref{def:flow}.
        \item Let $\Z_{\mathrm{init}} \in \R^{(h/r_1) \times (w/r_1) \times c}$ denote the initial input denoting the class condition.
        \item Let $\Z^1:=\Z_{\mathrm{init}} \in \R^{(h/r_1)\times(w/r_1) \times c}$.
    \end{itemize}
    Then, we define the architecture of FlowAR in the inference pipeline as follows:
    \begin{itemize}
        \item Layer 1: Given the initial token $\Z^1$, we compute
        \begin{align*}
            &~ s_1 = \mathsf{FFN}_1 \circ \mathsf{Attn}_1 (\Z^1) \in \R^{(h/r_1) \times (w/r_1) \times c}\\
            &~ \wh{s}_1 = \mathsf{NN}_1(F_t^1,s_1,t_1)
        \end{align*}
        \item Layer 2: Given the initial token $\Z^1$ and output of the first layer $\wh{s}_1$. Let $\Z^2$ be the reshape of the input sequence $\Z_{\mathrm{init}}, \phi_{\mathrm{up},1}(\wh{s}_1,a)$ into the tensor of size $(\sum_{i=1}^2 h/r_i) \times (\sum_{i=1}^2 w/r_i) \times c$. Then we compute
        \begin{align*}
            &~ s_2 = \mathsf{FFN}_2 \circ \mathsf{Attn}_2 (\Z^2)_{h/r_1:\sum_{i=1}^2 h/r_i,w/r_1 :\sum_{i=1}^2 w/r_i,0:c}\\
            &~ \wh{s}_2 = \mathsf{NN}_2(F_t^2,s_2,t_2)
        \end{align*}
        \item  Layer  $i \in [K]\setminus \{1,2\}$: Given the initial token $\Z^1$ and the output of the first $i-1$ layer $\wh{s}_1,\dots,\wh{s}_{i-1}$. Let $\Z^i$ be the reshape of the input sequence $\Z_{\mathrm{init}}, \phi_{\mathrm{up},1}(\wh{s}_1), \dots, \phi_{\mathrm{up},i-1}(\wh{s}_{i-1})$ into the tensor of size $(\sum_{j=1}^i h/r_j) \times (\sum_{j=1}^i w/r_j) \times c$. Then we compute
        \begin{align*}
            &~ s_i = \mathsf{FFN}_i \circ \mathsf{Attn}_i (\Z^i)_{\sum_{j=1}^{i-1} h/r_j : \sum_{j=1}^i h/r_j , \sum_{j=1}^{i-1} w/r_j : \sum_{j=1}^i w/r_j,0:c}\\
            &~ \wh{s}_i = \mathsf{NN}_i(F_t^i,s_i,t_i)
        \end{align*}
        Then the final output of FlowAR is $\wh{s}_K$.
    \end{itemize}
\end{definition}

\section{More Related Work}\label{sec:more_work}
In this section, we introduce more related work. 

\paragraph{Theoretical Machine Learning.}
Our work also takes inspiration from the following Machine Learning Theory work. Some works analyze the expressiveness of a neural network using the theory of circuit complexity~\cite{lls+25_gnn,kll+25_var_tc0,lls+24_rope_tensor_tc0,cll+24_mamba,cll+24_rope}. Some works optimize the algorithms that can accelerate the training of a neural network~\cite{llsz24,klsz24,dlms24,dswy22_coreset,haochen3,haochen4,dms23_spar,cll+25_deskreject,sy23,swyy23,lss+22,lsx+22,hst+22,hsw+22,hst+20,bsy23,dsy23,syyz23_weighted,gsy23_coin,gsy23_hyper,gsyz23,gswy23,syzz24,lsw+24,lsxy24,hsk+24,hlsl24}. Some works analyze neural networks via regressions~\cite{cll+24_icl,gms23,lsz23_exp,gsx23,ssz23_tradeoff,css+23,syyz23_ellinf,syz23,swy23,syz23_quantum,lls+25_grok}. Some works use reinforcement learning to optimize the neural networks~\cite{haochen1,haochen2,yunfan1,yunfan2,yunfan3,yunfan4,lswy23}. Some works optimize the attention mechanisms~\cite{sxy23,lls+24_conv}.

\paragraph{Accelerating Attention Mechanisms.}
The attention mechanism, with its quadratic computational complexity concerning context length, encounters increasing challenges as sequence lengths grow in modern large language models~\cite{gpto1,llama3_blog,claude3_pdf}. To address this limitation, polynomial kernel approximation methods \citep{aa22} have been introduced, leveraging low-rank approximations to efficiently approximate the attention matrix. These methods significantly enhance computation speed, allowing a single attention layer to perform both training and inference with nearly linear time complexity \citep{as23, as24b}. Moreover, these techniques can be extended to advanced attention mechanisms, such as tensor attention, while retaining almost linear time complexity for both training and inference \cite{as24_iclr}.~\cite{kll+25} provides an almost linear time algorithm to accelerate the inference of VAR Transformer. Other innovations include RoPE-based attention mechanisms~\cite{as24_rope,chl+24_rope} and differentially private cross-attention approaches~\cite{lssz24_dp}. Alternative strategies, such as the conv-basis method proposed in \cite{lls+24_conv}, present additional opportunities to accelerate attention computations, offering complementary solutions to this critical bottleneck. Additionally, various studies explore pruning-based methods to expedite attention mechanisms \cite{lls+24_prune,cls+24,llss24_sparse,ssz+25_prune,ssz+25_dit,hyw+23,whl+24,xhh+24,ssz+25_prune}.

\paragraph{Gradient Approximation.}
The low-rank approximation is a widely utilized approach for optimizing transformer training by reducing computational complexity \cite{lss+24,lssz24_tat,as24b,hwsl24,cls+24,lss+24_grad}. Building on the low-rank framework introduced in \cite{as23}, which initially focused on forward attention computation, \cite{as24b} extends this method to approximate attention gradients, effectively lowering the computational cost of gradient calculations. The study in \cite{lss+24} further expands this low-rank gradient approximation to multi-layer transformers, showing that backward computations in such architectures can achieve nearly linear time complexity. Additionally, \cite{lssz24_tat} generalizes the approach of \cite{as24b} to tensor-based attention models, utilizing forward computation results from \cite{as24_iclr} to enable efficient training of tensorized attention mechanisms. Lastly, \cite{hwsl24} applies low-rank approximation techniques during the training of Diffusion Transformers (DiTs), demonstrating the adaptability of these methods across various transformer-based architectures.

%% file: main.bbl
\newcommand{\etalchar}[1]{$^{#1}$}
\begin{thebibliography}{VdOKE{\etalchar{+}}16}

\bibitem[AA22]{aa22}
Amol Aggarwal and Josh Alman.
\newblock Optimal-degree polynomial approximations for exponentials and gaussian kernel density estimation.
\newblock In {\em Proceedings of the 37th Computational Complexity Conference}, pages 1--23, 2022.

\bibitem[ADTK23]{adtk23}
Silas Alberti, Niclas Dern, Laura Thesing, and Gitta Kutyniok.
\newblock Sumformer: Universal approximation for efficient transformers.
\newblock In {\em Topological, Algebraic and Geometric Learning Workshops 2023}, pages 72--86. PMLR, 2023.

\bibitem[AI24]{llama3_blog}
Meta AI.
\newblock Introducing meta llama 3: The most capable openly available llm to date, 2024.
\newblock \url{https://ai.meta.com/blog/meta-llama-3/}.

\bibitem[Ant24]{claude3_pdf}
Anthropic.
\newblock The claude 3 model family: Opus, sonnet, haiku, 2024.
\newblock \url{https://www-cdn.anthropic.com/de8ba9b01c9ab7cbabf5c33b80b7bbc618857627/Model_Card_Claude_3.pdf}.

\bibitem[AS23]{as23}
Josh Alman and Zhao Song.
\newblock Fast attention requires bounded entries.
\newblock {\em Advances in Neural Information Processing Systems}, 36, 2023.

\bibitem[AS24a]{as24_rope}
Josh Alman and Zhao Song.
\newblock Fast rope attention: Combining the polynomial method and fast fourier transform.
\newblock {\em manuscript}, 2024.

\bibitem[AS24b]{as24b}
Josh Alman and Zhao Song.
\newblock The fine-grained complexity of gradient computation for training large language models.
\newblock In {\em The Thirty-eighth Annual Conference on Neural Information Processing Systems}, 2024.

\bibitem[AS24c]{as24_iclr}
Josh Alman and Zhao Song.
\newblock How to capture higher-order correlations? generalizing matrix softmax attention to kronecker computation.
\newblock In {\em The Twelfth International Conference on Learning Representations}, 2024.

\bibitem[BNX{\etalchar{+}}23]{bnx+23}
Fan Bao, Shen Nie, Kaiwen Xue, Yue Cao, Chongxuan Li, Hang Su, and Jun Zhu.
\newblock All are worth words: A vit backbone for diffusion models.
\newblock In {\em Proceedings of the IEEE/CVF conference on computer vision and pattern recognition}, pages 22669--22679, 2023.

\bibitem[BSY23]{bsy23}
Song Bian, Zhao Song, and Junze Yin.
\newblock Federated empirical risk minimization via second-order method.
\newblock {\em arXiv preprint arXiv:2305.17482}, 2023.

\bibitem[CGL{\etalchar{+}}25a]{cgl+25_text}
Yuefan Cao, Chengyue Gong, Xiaoyu Li, Yingyu Liang, Zhizhou Sha, Zhenmei Shi, and Zhao Song.
\newblock Richspace: Enriching text-to-video prompt space via text embedding interpolation.
\newblock {\em arXiv preprint arXiv:2501.09982}, 2025.

\bibitem[CGL{\etalchar{+}}25b]{cgl+25_high}
Bo~Chen, Chengyue Gong, Xiaoyu Li, Yingyu Liang, Zhizhou Sha, Zhenmei Shi, Zhao Song, and Mingda Wan.
\newblock High-order matching for one-step shortcut diffusion models.
\newblock {\em arXiv preprint arXiv:2502.00688}, 2025.

\bibitem[CHL{\etalchar{+}}24]{chl+24_rope}
Yifang Chen, Jiayan Huo, Xiaoyu Li, Yingyu Liang, Zhenmei Shi, and Zhao Song.
\newblock Fast gradient computation for rope attention in almost linear time.
\newblock {\em arXiv preprint arXiv:2412.17316}, 2024.

\bibitem[CLL{\etalchar{+}}24a]{cll+24_rope}
Bo~Chen, Xiaoyu Li, Yingyu Liang, Jiangxuan Long, Zhenmei Shi, and Zhao Song.
\newblock Circuit complexity bounds for rope-based transformer architecture.
\newblock {\em arXiv preprint arXiv:2411.07602}, 2024.

\bibitem[CLL{\etalchar{+}}24b]{cll+24_icl}
Bo~Chen, Xiaoyu Li, Yingyu Liang, Zhenmei Shi, and Zhao Song.
\newblock Bypassing the exponential dependency: Looped transformers efficiently learn in-context by multi-step gradient descent.
\newblock {\em arXiv preprint arXiv:2410.11268}, 2024.

\bibitem[CLL{\etalchar{+}}24c]{cll+24_mamba}
Yifang Chen, Xiaoyu Li, Yingyu Liang, Zhenmei Shi, and Zhao Song.
\newblock The computational limits of state-space models and mamba via the lens of circuit complexity.
\newblock {\em arXiv preprint arXiv:2412.06148}, 2024.

\bibitem[CLL{\etalchar{+}}25]{cll+25_deskreject}
Yuefan Cao, Xiaoyu Li, Yingyu Liang, Zhizhou Sha, Zhenmei Shi, Zhao Song, and Jiahao Zhang.
\newblock Dissecting submission limit in desk-rejections: A mathematical analysis of fairness in ai conference policies.
\newblock {\em arXiv preprint arXiv:2502.00690}, 2025.

\bibitem[CLS{\etalchar{+}}24]{cls+24}
Bo~Chen, Yingyu Liang, Zhizhou Sha, Zhenmei Shi, and Zhao Song.
\newblock Hsr-enhanced sparse attention acceleration.
\newblock {\em arXiv preprint arXiv:2410.10165}, 2024.

\bibitem[CMRA18]{cmr+18}
Xi~Chen, Nikhil Mishra, Mostafa Rohaninejad, and Pieter Abbeel.
\newblock Pixelsnail: An improved autoregressive generative model.
\newblock In {\em International conference on machine learning}, pages 864--872. PMLR, 2018.

\bibitem[CSS{\etalchar{+}}23]{css+23}
Xiang Chen, Zhao Song, Baocheng Sun, Junze Yin, and Danyang Zhuo.
\newblock Query complexity of active learning for function family with nearly orthogonal basis.
\newblock {\em arXiv preprint arXiv:2306.03356}, 2023.

\bibitem[DLMS24]{dlms24}
Yichuan Deng, Zhihang Li, Sridhar Mahadevan, and Zhao Song.
\newblock Zero-th order algorithm for softmax attention optimization.
\newblock In {\em 2024 IEEE International Conference on Big Data (BigData)}, pages 24--33. IEEE, 2024.

\bibitem[DMS23]{dms23_spar}
Yichuan Deng, Sridhar Mahadevan, and Zhao Song.
\newblock Randomized and deterministic attention sparsification algorithms for over-parameterized feature dimension.
\newblock {\em arXiv preprint arXiv:2304.04397}, 2023.

\bibitem[DSWY22]{dswy22_coreset}
Yichuan Deng, Zhao Song, Yitan Wang, and Yuanyuan Yang.
\newblock A nearly optimal size coreset algorithm with nearly linear time.
\newblock {\em arXiv preprint arXiv:2210.08361}, 2022.

\bibitem[DSY23]{dsy23}
Yichuan Deng, Zhao Song, and Junze Yin.
\newblock Faster robust tensor power method for arbitrary order.
\newblock {\em arXiv preprint arXiv:2306.00406}, 2023.

\bibitem[DYH{\etalchar{+}}21]{dyh+21}
Ming Ding, Zhuoyi Yang, Wenyi Hong, Wendi Zheng, Chang Zhou, Da~Yin, Junyang Lin, Xu~Zou, Zhou Shao, Hongxia Yang, et~al.
\newblock Cogview: Mastering text-to-image generation via transformers.
\newblock {\em Advances in neural information processing systems}, 34:19822--19835, 2021.

\bibitem[DZHT22]{dzht22}
Ming Ding, Wendi Zheng, Wenyi Hong, and Jie Tang.
\newblock Cogview2: Faster and better text-to-image generation via hierarchical transformers.
\newblock {\em Advances in Neural Information Processing Systems}, 35:16890--16902, 2022.

\bibitem[ERO21]{ero21}
Patrick Esser, Robin Rombach, and Bjorn Ommer.
\newblock Taming transformers for high-resolution image synthesis.
\newblock In {\em Proceedings of the IEEE/CVF conference on computer vision and pattern recognition}, pages 12873--12883, 2021.

\bibitem[GMS23]{gms23}
Yeqi Gao, Sridhar Mahadevan, and Zhao Song.
\newblock An over-parameterized exponential regression.
\newblock {\em arXiv preprint arXiv:2303.16504}, 2023.

\bibitem[GSWY23]{gswy23}
Yeqi Gao, Zhao Song, Weixin Wang, and Junze Yin.
\newblock A fast optimization view: Reformulating single layer attention in llm based on tensor and svm trick, and solving it in matrix multiplication time.
\newblock {\em arXiv preprint arXiv:2309.07418}, 2023.

\bibitem[GSX23]{gsx23}
Yeqi Gao, Zhao Song, and Shenghao Xie.
\newblock In-context learning for attention scheme: from single softmax regression to multiple softmax regression via a tensor trick.
\newblock {\em arXiv preprint arXiv:2307.02419}, 2023.

\bibitem[GSY23a]{gsy23_coin}
Yeqi Gao, Zhao Song, and Junze Yin.
\newblock Gradientcoin: A peer-to-peer decentralized large language models.
\newblock {\em arXiv preprint arXiv:2308.10502}, 2023.

\bibitem[GSY23b]{gsy23_hyper}
Yeqi Gao, Zhao Song, and Junze Yin.
\newblock An iterative algorithm for rescaled hyperbolic functions regression.
\newblock {\em arXiv preprint arXiv:2305.00660}, 2023.

\bibitem[GSYZ24]{gsyz23}
Yuzhou Gu, Zhao Song, Junze Yin, and Lichen Zhang.
\newblock Low rank matrix completion via robust alternating minimization in nearly linear time.
\newblock In {\em The Twelfth International Conference on Learning Representations}, 2024.

\bibitem[HJA20]{hja20}
Jonathan Ho, Ajay Jain, and Pieter Abbeel.
\newblock Denoising diffusion probabilistic models.
\newblock {\em Advances in neural information processing systems}, 33:6840--6851, 2020.

\bibitem[HL24]{hl24}
Alexander Havrilla and Wenjing Liao.
\newblock Understanding scaling laws with statistical and approximation theory for transformer neural networks on intrinsically low-dimensional data.
\newblock In {\em The Thirty-eighth Annual Conference on Neural Information Processing Systems}, 2024.

\bibitem[HLSL24]{hlsl24}
Jerry Yao-Chieh Hu, Thomas Lin, Zhao Song, and Han Liu.
\newblock On computational limits of modern hopfield models: A fine-grained complexity analysis.
\newblock In {\em Forty-first International Conference on Machine Learning (ICML)}, 2024.

\bibitem[HSC{\etalchar{+}}22]{hsc+22}
Jonathan Ho, Chitwan Saharia, William Chan, David~J Fleet, Mohammad Norouzi, and Tim Salimans.
\newblock Cascaded diffusion models for high fidelity image generation.
\newblock {\em Journal of Machine Learning Research}, 23(47):1--33, 2022.

\bibitem[HSK{\etalchar{+}}24]{hsk+24}
Jerry Yao-Chieh Hu, Maojiang Su, En-Jui Kuo, Zhao Song, and Han Liu.
\newblock Computational limits of low-rank adaptation (lora) for transformer-based models.
\newblock {\em arXiv preprint arXiv:2406.03136}, 2024.

\bibitem[HST{\etalchar{+}}20]{hst+20}
Baihe Huang, Zhao Song, Runzhou Tao, Junze Yin, Ruizhe Zhang, and Danyang Zhuo.
\newblock Instahide's sample complexity when mixing two private images.
\newblock {\em arXiv preprint arXiv:2011.11877}, 2020.

\bibitem[HST{\etalchar{+}}22]{hst+22}
Hang Hu, Zhao Song, Runzhou Tao, Zhaozhuo Xu, Junze Yin, and Danyang Zhuo.
\newblock Sublinear time algorithm for online weighted bipartite matching.
\newblock {\em arXiv preprint arXiv:2208.03367}, 2022.

\bibitem[HSW{\etalchar{+}}22]{hsw+22}
Baihe Huang, Zhao Song, Omri Weinstein, Junze Yin, Hengjie Zhang, and Ruizhe Zhang.
\newblock A dynamic fast gaussian transform.
\newblock {\em arXiv preprint arXiv:2202.12329}, 2022.

\bibitem[HWG{\etalchar{+}}24]{hwg+24}
Jerry Yao-Chieh Hu, Wei-Po Wang, Ammar Gilani, Chenyang Li, Zhao Song, and Han Liu.
\newblock Fundamental limits of prompt tuning transformers: Universality, capacity and efficiency.
\newblock {\em arXiv preprint arXiv:2411.16525}, 2024.

\bibitem[HWL{\etalchar{+}}24]{hwl+24}
Jerry Yao-Chieh Hu, Weimin Wu, Yi-Chen Lee, Yu-Chao Huang, Minshuo Chen, and Han Liu.
\newblock On statistical rates of conditional diffusion transformers: Approximation, estimation and minimax optimality.
\newblock {\em arXiv preprint arXiv:2411.17522}, 2024.

\bibitem[HWSL24]{hwsl24}
Jerry Yao-Chieh Hu, Weimin Wu, Zhao Song, and Han Liu.
\newblock On statistical rates and provably efficient criteria of latent diffusion transformers (dits).
\newblock {\em arXiv preprint arXiv:2407.01079}, 2024.

\bibitem[HYW{\etalchar{+}}24]{hyw+23}
Jerry Yao-Chieh Hu, Donglin Yang, Dennis Wu, Chenwei Xu, Bo-Yu Chen, and Han Liu.
\newblock On sparse modern hopfield model.
\newblock {\em Advances in Neural Information Processing Systems}, 36, 2024.

\bibitem[JL23]{jl23}
Haotian Jiang and Qianxiao Li.
\newblock Approximation theory of transformer networks for sequence modeling.
\newblock {\em arXiv preprint arXiv:2305.18475}, 2023.

\bibitem[KKM22]{kkm22}
Junghwan Kim, Michelle Kim, and Barzan Mozafari.
\newblock Provable memorization capacity of transformers.
\newblock In {\em The Eleventh International Conference on Learning Representations}, 2022.

\bibitem[KLL{\etalchar{+}}25a]{kll+25}
Yekun Ke, Xiaoyu Li, Yingyu Liang, Zhizhou Sha, Zhenmei Shi, and Zhao Song.
\newblock On computational limits and provably efficient criteria of visual autoregressive models: A fine grained complexity analysis.
\newblock {\em arXiv preprint arXiv:2501.04377}, 2025.

\bibitem[KLL{\etalchar{+}}25b]{kll+25_var_tc0}
Yekun Ke, Xiaoyu Li, Yingyu Liang, Zhenmei Shi, and Zhao Song.
\newblock Circuit complexity bounds for visual autoregressive model.
\newblock {\em arXiv preprint arXiv:2501.04299}, 2025.

\bibitem[KLSZ24]{klsz24}
Yekun Ke, Xiaoyu Li, Zhao Song, and Tianyi Zhou.
\newblock Faster sampling algorithms for polytopes with small treewidth.
\newblock In {\em 2024 IEEE International Conference on Big Data (BigData)}, pages 44--53. IEEE, 2024.

\bibitem[KS24]{ks24}
Tokio Kajitsuka and Issei Sato.
\newblock Are transformers with one layer self-attention using low-rank weight matrices universal approximators?
\newblock In {\em The Twelfth International Conference on Learning Representations}, 2024.

\bibitem[LKK{\etalchar{+}}22]{lkk+22}
Doyup Lee, Chiheon Kim, Saehoon Kim, Minsu Cho, and Wook-Shin Han.
\newblock Autoregressive image generation using residual quantization.
\newblock In {\em Proceedings of the IEEE/CVF Conference on Computer Vision and Pattern Recognition}, pages 11523--11532, 2022.

\bibitem[LLL{\etalchar{+}}25]{lll+25_loop}
Xiaoyu Li, Yingyu Liang, Jiangxuan Long, Zhenmei Shi, Zhao Song, and Zhen Zhuang.
\newblock Neural algorithmic reasoning for hypergraphs with looped transformers.
\newblock {\em arXiv preprint arXiv:2501.10688}, 2025.

\bibitem[LLS{\etalchar{+}}24a]{lls+24_rope_tensor_tc0}
Xiaoyu Li, Yingyu Liang, Zhenmei Shi, Zhao Song, and Mingda Wan.
\newblock Theoretical constraints on the expressive power of rope-based tensor attention transformers.
\newblock {\em arXiv preprint arXiv:2412.18040}, 2024.

\bibitem[LLS{\etalchar{+}}24b]{lls+24_conv}
Yingyu Liang, Heshan Liu, Zhenmei Shi, Zhao Song, and Junze Yin.
\newblock Conv-basis: A new paradigm for efficient attention inference and gradient computation in transformers.
\newblock {\em arXiv preprint arXiv:2405.05219}, 2024.

\bibitem[LLS{\etalchar{+}}24c]{lls+24_prune}
Yingyu Liang, Jiangxuan Long, Zhenmei Shi, Zhao Song, and Yufa Zhou.
\newblock Beyond linear approximations: A novel pruning approach for attention matrix.
\newblock {\em arXiv preprint arXiv:2410.11261}, 2024.

\bibitem[LLS{\etalchar{+}}25a]{lls+25_grok}
Chenyang Li, Yingyu Liang, Zhenmei Shi, Zhao Song, and Tianyi Zhou.
\newblock Fourier circuits in neural networks and transformers: A case study of modular arithmetic with multiple inputs.
\newblock In {\em International Conference on Artificial Intelligence and Statistics}, 2025.

\bibitem[LLS{\etalchar{+}}25b]{lls+25_gnn}
Xiaoyu Li, Yingyu Liang, Zhenmei Shi, Zhao Song, Wei Wang, and Jiahao Zhang.
\newblock On the computational capability of graph neural networks: A circuit complexity bound perspective.
\newblock {\em arXiv preprint arXiv:2501.06444}, 2025.

\bibitem[LLSS24]{llss24_sparse}
Xiaoyu Li, Yingyu Liang, Zhenmei Shi, and Zhao Song.
\newblock A tighter complexity analysis of sparsegpt.
\newblock {\em arXiv preprint arXiv:2408.12151}, 2024.

\bibitem[LLSZ24]{llsz24}
Xiaoyu Li, Jiangxuan Long, Zhao Song, and Tianyi Zhou.
\newblock Fast second-order method for neural networks under small treewidth setting.
\newblock In {\em 2024 IEEE International Conference on Big Data (BigData)}, pages 1029--1038. IEEE, 2024.

\bibitem[LLWY24]{yunfan4}
Junyan Liu, Yunfan Li, Ruosong Wang, and Lin Yang.
\newblock Uniform last-iterate guarantee for bandits and reinforcement learning.
\newblock In {\em The Thirty-eighth Annual Conference on Neural Information Processing Systems}, 2024.

\bibitem[LLY24]{yunfan3}
Junyan Liu, Yunfan Li, and Lin Yang.
\newblock Achieving near-optimal regret for bandit algorithms with uniform last-iterate guarantee.
\newblock {\em arXiv preprint arXiv:2402.12711}, 2024.

\bibitem[LSS{\etalchar{+}}22]{lss+22}
Jiehao Liang, Somdeb Sarkhel, Zhao Song, Chenbo Yin, Junze Yin, and Danyang Zhuo.
\newblock A faster $ k $-means++ algorithm.
\newblock {\em arXiv preprint arXiv:2211.15118}, 2022.

\bibitem[LSS{\etalchar{+}}24a]{lss+24}
Yingyu Liang, Zhizhou Sha, Zhenmei Shi, Zhao Song, and Yufa Zhou.
\newblock Looped relu mlps may be all you need as practical programmable computers.
\newblock {\em arXiv preprint arXiv:2410.09375}, 2024.

\bibitem[LSS{\etalchar{+}}24b]{lss+24_grad}
Yingyu Liang, Zhizhou Sha, Zhenmei Shi, Zhao Song, and Yufa Zhou.
\newblock Multi-layer transformers gradient can be approximated in almost linear time.
\newblock {\em arXiv preprint arXiv:2408.13233}, 2024.

\bibitem[LSSS24]{lsss24}
Yingyu Liang, Zhizhou Sha, Zhenmei Shi, and Zhao Song.
\newblock Differential privacy mechanisms in neural tangent kernel regression.
\newblock {\em arXiv preprint arXiv:2407.13621}, 2024.

\bibitem[LSSZ24a]{lssz24_dp}
Yingyu Liang, Zhenmei Shi, Zhao Song, and Yufa Zhou.
\newblock Differential privacy of cross-attention with provable guarantee.
\newblock {\em arXiv preprint arXiv:2407.14717}, 2024.

\bibitem[LSSZ24b]{lssz24_tat}
Yingyu Liang, Zhenmei Shi, Zhao Song, and Yufa Zhou.
\newblock Tensor attention training: Provably efficient learning of higher-order transformers.
\newblock {\em arXiv preprint arXiv:2405.16411}, 2024.

\bibitem[LSW{\etalchar{+}}24]{lsw+24}
Zhihang Li, Zhao Song, Weixin Wang, Junze Yin, and Zheng Yu.
\newblock How to inverting the leverage score distribution?
\newblock {\em arXiv preprint arXiv:2404.13785}, 2024.

\bibitem[LSWY23]{lswy23}
Zhihang Li, Zhao Song, Zifan Wang, and Junze Yin.
\newblock Local convergence of approximate newton method for two layer nonlinear regression.
\newblock {\em arXiv preprint arXiv:2311.15390}, 2023.

\bibitem[LSX{\etalchar{+}}22]{lsx+22}
Jiehao Liang, Zhao Song, Zhaozhuo Xu, Junze Yin, and Danyang Zhuo.
\newblock Dynamic maintenance of kernel density estimation data structure: From practice to theory.
\newblock {\em arXiv preprint arXiv:2208.03915}, 2022.

\bibitem[LSXY24]{lsxy24}
Chenyang Li, Zhao Song, Zhaoxing Xu, and Junze Yin.
\newblock Inverting the leverage score gradient: An efficient approximate newton method.
\newblock {\em arXiv preprint arXiv:2408.11267}, 2024.

\bibitem[LSZ23]{lsz23_exp}
Zhihang Li, Zhao Song, and Tianyi Zhou.
\newblock Solving regularized exp, cosh and sinh regression problems.
\newblock {\em arXiv preprint arXiv:2303.15725}, 2023.

\bibitem[LT24]{llama3_arxiv}
AI~@~Meta Llama~Team.
\newblock The llama 3 herd of models.
\newblock {\em arXiv preprint arXiv:2407.21783}, 2024.

\bibitem[LWCY23]{yunfan1}
Yunfan Li, Yiran Wang, Yu~Cheng, and Lin Yang.
\newblock Low-switching policy gradient with exploration via online sensitivity sampling.
\newblock In {\em International Conference on Machine Learning}, pages 19995--20034. PMLR, 2023.

\bibitem[LY24]{yunfan2}
Yunfan Li and Lin Yang.
\newblock On the model-misspecification in reinforcement learning.
\newblock In {\em International Conference on Artificial Intelligence and Statistics}, pages 2764--2772. PMLR, 2024.

\bibitem[LZB{\etalchar{+}}22]{lzb+22}
C~Lu, Y~Zhou, F~Bao, J~Chen, and C~Li.
\newblock A fast ode solver for diffusion probabilistic model sampling in around 10 steps.
\newblock {\em Proc. Adv. Neural Inf. Process. Syst., New Orleans, United States}, pages 1--31, 2022.

\bibitem[LZW{\etalchar{+}}24]{lzw+24}
Chengyi Liu, Jiahao Zhang, Shijie Wang, Wenqi Fan, and Qing Li.
\newblock Score-based generative diffusion models for social recommendations.
\newblock {\em arXiv preprint arXiv:2412.15579}, 2024.

\bibitem[Ope24]{gpto1}
OpenAI.
\newblock Introducing openai o1-preview.
\newblock \url{ https://openai.com/index/introducing-openai-o1-preview/}, 2024.
\newblock Accessed: September 12.

\bibitem[PJBM22]{pjbm22}
Ben Poole, Ajay Jain, Jonathan~T Barron, and Ben Mildenhall.
\newblock Dreamfusion: Text-to-3d using 2d diffusion.
\newblock {\em arXiv preprint arXiv:2209.14988}, 2022.

\bibitem[PX23]{px23}
William Peebles and Saining Xie.
\newblock Scalable diffusion models with transformers.
\newblock In {\em Proceedings of the IEEE/CVF International Conference on Computer Vision}, pages 4195--4205, 2023.

\bibitem[RBL{\etalchar{+}}22]{rbl+22}
Robin Rombach, Andreas Blattmann, Dominik Lorenz, Patrick Esser, and Bj{\"o}rn Ommer.
\newblock High-resolution image synthesis with latent diffusion models.
\newblock In {\em Proceedings of the IEEE/CVF conference on computer vision and pattern recognition}, pages 10684--10695, 2022.

\bibitem[RVdOV19]{rvav19}
Ali Razavi, Aaron Van~den Oord, and Oriol Vinyals.
\newblock Generating diverse high-fidelity images with vq-vae-2.
\newblock {\em Advances in neural information processing systems}, 32, 2019.

\bibitem[RYH{\etalchar{+}}24]{ryh+24}
Sucheng Ren, Qihang Yu, Ju~He, Xiaohui Shen, Alan Yuille, and Liang-Chieh Chen.
\newblock Flowar: Scale-wise autoregressive image generation meets flow matching.
\newblock {\em arXiv preprint arXiv:2412.15205}, 2024.

\bibitem[SE19]{se19}
Yang Song and Stefano Ermon.
\newblock Generative modeling by estimating gradients of the data distribution.
\newblock {\em Advances in neural information processing systems}, 32, 2019.

\bibitem[SME20]{sme20}
Jiaming Song, Chenlin Meng, and Stefano Ermon.
\newblock Denoising diffusion implicit models.
\newblock {\em arXiv preprint arXiv:2010.02502}, 2020.

\bibitem[SSZ23]{ssz23_tradeoff}
Ritwik Sinha, Zhao Song, and Tianyi Zhou.
\newblock A mathematical abstraction for balancing the trade-off between creativity and reality in large language models.
\newblock {\em arXiv preprint arXiv:2306.02295}, 2023.

\bibitem[SSZ{\etalchar{+}}25a]{ssz+25_dit}
Xuan Shen, Zhao Song, Yufa Zhou, Bo~Chen, Yanyu Li, Yifan Gong, Kai Zhang, Hao Tan, Jason Kuen, Henghui Ding, Zhihao Shu, Wei Niu, Pu~Zhao, Yanzhi Wang, and Jiuxiang Gu.
\newblock Lazydit: Lazy learning for the acceleration of diffusion transformers.
\newblock In {\em Proceedings of the AAAI Conference on Artificial Intelligence}, 2025.

\bibitem[SSZ{\etalchar{+}}25b]{ssz+25_prune}
Xuan Shen, Zhao Song, Yufa Zhou, Bo~Chen, Jing Liu, Ruiyi Zhang, Ryan~A. Rossi, Hao Tan, Tong Yu, Xiang Chen, Yufan Zhou, Tong Sun, Pu~Zhao, Yanzhi Wang, and Jiuxiang Gu.
\newblock Numerical pruning for efficient autoregressive models.
\newblock In {\em Proceedings of the AAAI Conference on Artificial Intelligence}, 2025.

\bibitem[SWY23]{swy23}
Zhao Song, Weixin Wang, and Junze Yin.
\newblock A unified scheme of resnet and softmax.
\newblock {\em arXiv preprint arXiv:2309.13482}, 2023.

\bibitem[SWYY23]{swyy23}
Zhao Song, Weixin Wang, Chenbo Yin, and Junze Yin.
\newblock Fast and efficient matching algorithm with deadline instances.
\newblock {\em arXiv preprint arXiv:2305.08353}, 2023.

\bibitem[SXY23]{sxy23}
Zhao Song, Guangyi Xu, and Junze Yin.
\newblock The expressibility of polynomial based attention scheme.
\newblock {\em arXiv preprint arXiv:2310.20051}, 2023.

\bibitem[SY23]{sy23}
Zhao Song and Chiwun Yang.
\newblock An automatic learning rate schedule algorithm for achieving faster convergence and steeper descent.
\newblock {\em arXiv preprint arXiv:2310.11291}, 2023.

\bibitem[SYYZ23]{syyz23_ellinf}
Zhao Song, Mingquan Ye, Junze Yin, and Lichen Zhang.
\newblock A nearly-optimal bound for fast regression with $\ell_\infty$ guarantee.
\newblock In {\em International Conference on Machine Learning}, pages 32463--32482. PMLR, 2023.

\bibitem[SYYZ25]{syyz23_weighted}
Zhao Song, Mingquan Ye, Junze Yin, and Lichen Zhang.
\newblock Efficient alternating minimization with applications to weighted low rank approximation.
\newblock In {\em The Thirteenth International Conference on Learning Representations}, 2025.

\bibitem[SYZ23]{syz23_quantum}
Zhao Song, Junze Yin, and Ruizhe Zhang.
\newblock Revisiting quantum algorithms for linear regressions: Quadratic speedups without data-dependent parameters.
\newblock {\em arXiv preprint arXiv:2311.14823}, 2023.

\bibitem[SYZ24]{syz23}
Zhao Song, Junze Yin, and Lichen Zhang.
\newblock Solving attention kernel regression problem via pre-conditioner.
\newblock In {\em International Conference on Artificial Intelligence and Statistics}, pages 208--216. PMLR, 2024.

\bibitem[SYZZ24]{syzz24}
Zhao Song, Junze Yin, Lichen Zhang, and Ruizhe Zhang.
\newblock Fast dynamic sampling for determinantal point processes.
\newblock In {\em International Conference on Artificial Intelligence and Statistics}, pages 244--252. PMLR, 2024.

\bibitem[TJY{\etalchar{+}}24]{tjy+24}
Keyu Tian, Yi~Jiang, Zehuan Yuan, Bingyue Peng, and Liwei Wang.
\newblock Visual autoregressive modeling: Scalable image generation via next-scale prediction.
\newblock {\em Advances in neural information processing systems (NeurIPS)}, 2024.

\bibitem[VdOKE{\etalchar{+}}16]{vke+16}
Aaron Van~den Oord, Nal Kalchbrenner, Lasse Espeholt, Oriol Vinyals, Alex Graves, et~al.
\newblock Conditional image generation with pixelcnn decoders.
\newblock {\em Advances in neural information processing systems}, 29, 2016.

\bibitem[VSP{\etalchar{+}}17]{vsp+17}
Ashish Vaswani, Noam Shazeer, Niki Parmar, Jakob Uszkoreit, Llion Jones, Aidan~N Gomez, {\L}ukasz Kaiser, and Illia Polosukhin.
\newblock Attention is all you need.
\newblock {\em Advances in neural information processing systems}, 30, 2017.

\bibitem[WCZ{\etalchar{+}}23]{wcz+23}
Yilin Wang, Zeyuan Chen, Liangjun Zhong, Zheng Ding, Zhizhou Sha, and Zhuowen Tu.
\newblock Dolfin: Diffusion layout transformers without autoencoder.
\newblock {\em arXiv preprint arXiv:2310.16305}, 2023.

\bibitem[WHL{\etalchar{+}}24]{whl+24}
Dennis Wu, Jerry Yao-Chieh Hu, Weijian Li, Bo-Yu Chen, and Han Liu.
\newblock {ST}anhop: Sparse tandem hopfield model for memory-enhanced time series prediction.
\newblock In {\em The Twelfth International Conference on Learning Representations (ICLR)}, 2024.

\bibitem[WLW{\etalchar{+}}24]{wlw+24}
Zhengyi Wang, Cheng Lu, Yikai Wang, Fan Bao, Chongxuan Li, Hang Su, and Jun Zhu.
\newblock Prolificdreamer: High-fidelity and diverse text-to-3d generation with variational score distillation.
\newblock {\em Advances in Neural Information Processing Systems}, 36, 2024.

\bibitem[WSD{\etalchar{+}}24]{wsd+24}
Zirui Wang, Zhizhou Sha, Zheng Ding, Yilin Wang, and Zhuowen Tu.
\newblock Tokencompose: Text-to-image diffusion with token-level supervision.
\newblock In {\em Proceedings of the IEEE/CVF Conference on Computer Vision and Pattern Recognition}, pages 8553--8564, 2024.

\bibitem[WXZ{\etalchar{+}}24]{wxz+24}
Yilin Wang, Haiyang Xu, Xiang Zhang, Zeyuan Chen, Zhizhou Sha, Zirui Wang, and Zhuowen Tu.
\newblock Omnicontrolnet: Dual-stage integration for conditional image generation.
\newblock In {\em Proceedings of the IEEE/CVF Conference on Computer Vision and Pattern Recognition}, pages 7436--7448, 2024.

\bibitem[XHH{\etalchar{+}}24]{xhh+24}
Chenwei Xu, Yu-Chao Huang, Jerry Yao-Chieh Hu, Weijian Li, Ammar Gilani, Hsi-Sheng Goan, and Han Liu.
\newblock Bishop: Bi-directional cellular learning for tabular data with generalized sparse modern hopfield model.
\newblock In {\em Forty-first International Conference on Machine Learning}, 2024.

\bibitem[XLC{\etalchar{+}}24]{xlc+24}
Haiyang Xu, Yu~Lei, Zeyuan Chen, Xiang Zhang, Yue Zhao, Yilin Wang, and Zhuowen Tu.
\newblock Bayesian diffusion models for 3d shape reconstruction.
\newblock In {\em Proceedings of the IEEE/CVF Conference on Computer Vision and Pattern Recognition}, pages 10628--10638, 2024.

\bibitem[XSG{\etalchar{+}}24]{xsg+24}
Zeyue Xue, Guanglu Song, Qiushan Guo, Boxiao Liu, Zhuofan Zong, Yu~Liu, and Ping Luo.
\newblock Raphael: Text-to-image generation via large mixture of diffusion paths.
\newblock {\em Advances in Neural Information Processing Systems}, 36, 2024.

\bibitem[YBR{\etalchar{+}}20]{ybr+20}
Chulhee Yun, Srinadh Bhojanapalli, Ankit~Singh Rawat, Sashank Reddi, and Sanjiv Kumar.
\newblock Are transformers universal approximators of sequence-to-sequence functions?
\newblock In {\em International Conference on Learning Representations (ICLR)}, 2020.

\bibitem[ZCY23]{haochen1}
Haochen Zhang, Xi~Chen, and Lin~F Yang.
\newblock Adaptive liquidity provision in uniswap v3 with deep reinforcement learning.
\newblock {\em arXiv preprint arXiv:2309.10129}, 2023.

\bibitem[ZCZ{\etalchar{+}}25]{haochen2}
Zhi Zhang, Chris Chow, Yasi Zhang, Yanchao Sun, Haochen Zhang, Eric~Hanchen Jiang, Han Liu, Furong Huang, Yuchen Cui, and Oscar Hernan~Madrid Padilla.
\newblock Statistical guarantees for lifelong reinforcement learning using pac-bayesian theory.
\newblock In {\em International Conference on Artificial Intelligence and Statistics}, 2025.

\bibitem[ZLP{\etalchar{+}}23]{haochen4}
Haochen Zhang, Xingyu Lin, Sui Peng, Junjie Tang, Antonello Monti, et~al.
\newblock Surrogate-model-based sequential algorithm for weather-dependent probabilistic power flow with high calculation efficiency.
\newblock {\em Authorea Preprints}, 2023.

\bibitem[ZPT{\etalchar{+}}22]{haochen3}
Haochen Zhang, Zhiyun Peng, Junjie Tang, Ming Dong, Ke~Wang, and Wenyuan Li.
\newblock A multi-layer extreme learning machine refined by sparrow search algorithm and weighted mean filter for short-term multi-step wind speed forecasting.
\newblock {\em Sustainable Energy Technologies and Assessments}, 50:101698, 2022.

\end{thebibliography}
